\documentclass[11pt]{article}
\usepackage{microtype}
\usepackage{graphicx}
\usepackage{subcaption}
\usepackage{booktabs}
\usepackage{amsfonts}
\usepackage{amsmath}
\usepackage{amssymb}
\usepackage{mathtools}
\usepackage{amsthm}
\usepackage{makecell}
\usepackage{enumerate}
\usepackage{algorithm, algorithmic}

\usepackage{bookmark}
\usepackage{authblk}
\usepackage{hyperref,url,dsfont}
\usepackage{natbib}
\usepackage{fullpage}

\usepackage{float}
\usepackage{wrapfig}
\usepackage{caption}

\hypersetup{
    colorlinks,
    breaklinks,
    urlcolor = black,
    linkcolor = blue,
    citecolor = blue,
}

\usepackage[capitalize,noabbrev]{cleveref}

\theoremstyle{plain}
\newtheorem{theorem}{Theorem}[section]
\newtheorem{proposition}{Proposition}[section]
\newtheorem{lemma}{Lemma}[section]

\theoremstyle{definition}
\newtheorem{definition}{Definition}[section]
\newtheorem{assumption}{Assumption}[section]
\theoremstyle{remark}
\newtheorem{remark}{Remark}[section]

\DeclareFixedFont{\myfontb}{OT1}{ptm}{bx}{n}{11pt}

\def \altereq {\stackrel{\mathrm a}{=}}
\def \A {\mathbf A}

\def \bLambda {\mathbf\Lambda}
\def \bOmega {\mathbf\Omega}
\def \bSigma {\mathbf\Sigma}
\def \c {\mathbf c}

\def \d {\mathbf d}
\def \D {\mathbf D}
\def \Des {\mathbf{Des}}
\def \e {\mathrm e}
\def \E {\mathbb E}
\def \given {~|~}
\def \G {\mathcal G}
\def \I {\mathbb I}
\def \M {\mathbf M}
\def \N {\mathbf N}

\def \pa {\mathbf{pa}}
\def \Pa {\mathbf{Pa}}
\def \pred {\mathbf{pred}}
\def \Pred {\mathbf{Pred}}
\def \S {\mathcal S}

\def \v {\mathbf v}
\def \V {\mathbf V}

\def \W {\mathbf W}
\def \x {\mathbf x}
\def \X {\mathbf X}
\def \y {\mathbf y}
\def \Y {\mathbf Y}
\def \z {\mathbf z}
\def \Z {\mathbf Z}

\begin{document}

\title{Order-based Rehearsal Learning}

\author{%
  Yu-Xuan Tao$ ^{1,2}$~ Tian-Zuo Wang$ ^{1,2}$~ Zhi-Hua Zhou$ ^{1,2}$\\ 
  $^1$ National Key Laboratory for Novel Software Technology, Nanjing University, China\\
  $^2$ School of Artifcial Intelligence, Nanjing University, China\\
}
\date{}
\maketitle

\begin{abstract}

\noindent When a machine learning (ML) model forecasts an undesired event, one often seeks a decision to avoid it, known as the \emph{avoiding undesired future (AUF)} problem. 
Many \emph{rehearsal learning} methods have been proposed for AUF, but they rely on an underlying graph structure; learning such a graph from observational data is challenging and can incur substantial estimation error.
In this work, we demonstrate that the \emph{order} structure can be sufficient for AUF decision-making, and propose the first order-based rehearsal learning method. Although an order is less informative than a graph, it can be sufficient to identify the influence of decisions from observational data, suggesting that learning the entire graph is not always necessary.
To learn the order, we develop an information-theoretic method that imposes no restrictions on the form of structural functions or the type of noise distributions.
For AUF decision-making, we construct an order-based sampler to approximate the influence of decisions and, combined with a surrogate objective for maximizing the post-decision success probability, reduce the AUF task to a differentiable optimization problem.
Experiments show that our order learning method outperforms existing methods, and that our AUF approach not only surpasses methods relying on learned graphs or learned orders, but also matches or even exceeds oracle baselines that are given the true graph.

\end{abstract}

\section{Introduction}
Various machine learning (ML) models have demonstrated strong performance on prediction tasks \citep{DL4CV, DL4NLP, DL4recommender}. In many real-world scenarios, however, the goal is not merely to predict outcomes but to make better decisions. Specifically, when an ML model forecasts an undesired event, such as a decline in sales, one may demand a decision to avoid this, which is referred to as the \emph{avoiding undesired future (AUF)} problem \citep{Zhou2022}. To address this, several \emph{rehearsal learning} approaches have been proposed \citep{nips/QWZ2023, nips/Du2025MuR, aaai/Qin2025gradient, ijcai/Taol2025sequential}. They typically leverage an underlying graph structure to recommend a decision maximizing the probability that the outcome lies within a pre-specified desirable region.

However, the aforementioned methods assume access to a given graph structure that characterizes the relations among all variables, which rarely holds in real-world applications. A plausible solution is to learn the graph structure from the observational data, for which plenty of relevant methods have been developed. For example, in the absence of mutual influence relations and latent variables, a directed acyclic graph (DAG) is widely used to characterize the relations among the variables, and numerous causal discovery methods are proposed for learning a DAG from data \citep{causal_book_Pearl2000, causal_book_Spirtes2000, causal_survey_methods2019, causal_survey_methods2022}. Nevertheless, learning such a graph structure is often challenging and typically requires strong assumptions \citep{FCI2008, GES2002, LinGAM2006, nips/ANM2008, geometry_faithfulness2013, causal_representation_learning_KunZhang2024, causal_FengXie2024}.

Motivated by the inherent difficulty of graph learning, we revisit a fundamental question: is an explicit graph structure truly necessary for AUF decision-making? In many practical tasks, humans often make effective decisions without full knowledge of the underlying graph structure~\citep{Zhou2022}. This observation suggests that there may exist a structure that is easier to learn than a graph structure, while still being sufficient to support effective decision-making in AUF problems.
 
In this paper, we highlight the importance of the \emph{order} structure for decision-making in AUF problems, and further propose the first order-based rehearsal learning method. An order structure encodes substantially less information than a graph structure; it is therefore natural to expect that learning an order is easier than learning a graph. We show that the influence of decisions can be identified from the observational data using only the variable order, thus establishing the sufficiency of order structure for AUF decision-making.

To recover the variable order, we develop OLEM (\emph{Order Learning via conditional Entropy Maximization}). OLEM learns the order recursively by repeatedly identifying and removing the ``last'' variable from the remaining set of variables. Each step is achieved by maximizing a conditional entropy objective, which is simple to implement and readily compatible with a wide range of entropy estimators \citep{entropy_estimate_kl(1nn), entropy_estimate_ensemble, entropy_estimate_knn}. Compared with existing methods for learning the variable order~\citep{CAM2014, aistats/LISTEN2018, icml/SCORE2022, iclr/DiffAN2023, nips/CaPS2024}, OLEM provides soundness guarantees without requiring restrictive assumptions on a specific form of structural functions or the type of noise distributions. Moreover, OLEM can be implemented in polynomial time with respect to the variable number, offering an efficiency advantage over classical constraint-based discovery methods such as PC and FCI methods~\citep{causal_book_Spirtes2000, FCI2008}.

Building on the learned order, we further develop a decision-making method, OLEM-Rh, which searches for an action that maximizes the probability that the outcome falls within a desired region. Specifically, we employ the conditional flow model \citep{conditional_flows_ardizzone2019, conditional_flows_winkler2019} to construct a sampler for the post-decision distribution based on the order, enabling us to approximate the influence of candidate decisions. Combining with the surrogate objective proposed by \citet{aaai/Qin2025gradient} for maximizing the desired-region probability, our method reduces AUF decision-making to a differentiable optimization problem which can be solved using standard gradient-based optimization methods. Experiments validate that OLEM outperforms existing order learning methods, and further show that for AUF problems, OLEM-Rh surpasses pipelines that first learn a graph or learn an order using existing methods, while performing comparably to, or even better than, the oracle baselines that take the true graph structure as input.

Our main contributions are summarized as follows:
\begin{enumerate}[1.]
    \item We highlight the importance of the order structure for AUF problems. We show that the influence of decisions can be identified from observational data given only a variable order, implying that the order structure can be sufficient for AUF decision-making.
    \item We develop a recursive method to learn the variable order, which comes with soundness guarantees while requiring milder assumptions than relevant methods.
    \item We present a decision-making method based on the variable order by formulating a differentiable optimization problem using a post-decision distribution sampler and a surrogate objective, resulting in the first order-based rehearsal learning method.
\end{enumerate}

\section{Problem Formulation}
\label{Section: AUF formulation}

In this section, we introduce the data-generating model and notations formalizing the AUF problem. We assume there are no latent variables and no feedback, \emph{i.e.}, no cyclic or mutual influence relations.

In this paper, we consider a random vector $\V=(V_1,\dots,V_d)$ generated by the additive model
\begin{align}
    V_i = f_i(\Pa_i)+\varepsilon_i,~i\in[d]\triangleq\{1,\dots,d\}.
    \label{eq: SEM}
\end{align}
$\Pa_i\triangleq\{V_j~|~j~\text{is parent of}~i~\text{in}~\G\}$, where $\G$ is a directed acyclic graph with node set $[d]$. Noise variables $\varepsilon_i$ are independent.
Each $f_i$ is non-constant in every $V_j\in\Pa_i$ without loss of generality; else $\Pa_i$ can be redefined to satisfy this.
We denote the model above by a triple $\mathcal M=(\G, \{f_i\}_i, \{\varepsilon_i\}_i)$.

For an AUF decision-making task \citep{nips/QWZ2023}, the variables $V_1,\dots,V_d$ are separated into three disjoint parts $\mathbf{X},\mathbf{Y}$, and $\mathbf{Z}$. The random vector $\X=(X_i)_i$ denotes contextual variables, which are observed before decision-making, $\Y=(Y_i)_i$ denotes decision outcome variables, which the decision-makers wish to fall into a desired region $\S$, and are observed after the decision, and $\Z=(Z_i)_i$ denotes intermediate variables, which appear after observing $\X$ and before observing $\Y$. Note that, any variable in $\Z$ or $\Y$ cannot be a parent of a contextual variable $X_i$, as $X_i$ has been observed before the decisions. Moreover, part of the intermediate variables $\Z$ are alterable, denoted by $\Z_{\A}=(Z_{a_1},\dots,Z_{a_{|\A|}})$, where $\A=(a_i)_i$ is the index vector and $|\A|$ denotes its dimensionality. The goal is to recommend an alteration $\z_{\A}=(z_{a_1},\dots,z_{a_{|\A|}})$ of $\Z_{\A}$, so as to increase the probability that $\Y$ falls into the pre-specified desired region $\S$, given the context $\x$. Here, $\S$ is considered to be a convex polytope, i.e. $\S=\{\y\given\M\y\preceq\d\}$, where the matrix $\M$ and vector $\d$ are given. $\mathbf\Delta_{\A}$ denotes the feasible domain of $\Z_{\A}$. Formally, we aim to solve the following optimization problem:
\begin{align}
\begin{split}
    \max_{\z_{\A}}\quad &\Pr\{\Y\in\S\given\X=\x,~\Z_\A\altereq\z_\A\}\\
    \text{s.t.}\quad &~\z_{\A}\in \mathbf\Delta_{\A},
    \label{eq: AUF}
\end{split}
\end{align}
where $\Z_\A\altereq\z_\A$ means setting $\Z_\A$ to the alteration $\z_\A$. Note that we do not assume access to the model $\mathcal M$. Instead, we only have observational data $\D=\{\V^i\}_i$.

\section{Proposed Method}
\label{Section: proposed method}
In this section, we demonstrate the sufficiency of the order structure in AUF problems and detail our order-learning approach. Further, we propose a decision method exploiting an order-based sampler.

\subsection{Order Structure}
\label{Subsection: order sufficient}

The order structure is introduced in \autoref{Def: order}. Subsequently, we present \autoref{Prop: post-order distribution factorize}, which implies that given the order of the true graph $\G$, the influence of decisions in the AUF problem (\ref{eq: AUF}) can be identified from the observational distribution. This result demonstrates that the order structure can be sufficient for addressing AUF decision-making tasks.

\begin{definition}[Order]
    Given a graph $\G$ with node set $[d]$, a permutation $\pi=(\pi(1),\dots,\pi(d))$ of $1,2,\dots,d$ is called an order of $\G$, if and only if $\forall i,j\in[d]$, ``$j$ is a descendant of $i$ in $\G$'' implies ``$\pi^{-1}(i) < \pi^{-1}(j)$''. We write $\pi\sim\G$ to mean that $\pi$ is an order of $\G$.
    \label{Def: order}
\end{definition}

\begin{proposition}
    \label{Prop: post-order distribution factorize}
    Given a random vector $\V$ generated by the model $\mathcal M=(\G, \{f_i\}_i, \{\varepsilon_i\}_i)$, suppose $\V$ is separated into $\X,\Z,\Y$ with $\Z_\A\subseteq\Z$ according to the AUF setting described in $\autoref{Section: AUF formulation}$.
        Given an order $\pi\sim\G$, it holds that
        \begin{align*}
            &p(\v\given\X=\x,~\Z_\A\altereq\z_\A)\\
            &=\delta_{\X=\x}(\v)~\delta_{\Z_\A=\z_\A}(\v)\prod_{i}p(v_i\given \pred_i(\pi)),
        \end{align*}
        where the product is taken over the $V_i$-s that are not contained in $\X$ or $\Z_\A$. $\Pred_i(\pi)\triangleq\{V_j\given \pi^{-1}(j)<\pi^{-1}(i)\}$ denotes the predecessors of $V_i$ according to $\pi$, and $\pred_i(\pi)$ denotes an assignment to $\Pred_i(\pi)$. $\delta_{\X=\x}(\v)$ is a Dirac function satisfying: (1) $\delta_{\X=\x}(\v)=0$ if the sub-vector of $\v$ corresponding to $\X$ is not equal to $\x$, and (2) $\int\delta_{\X=\x}(\v)\d\v=1$. The Dirac function $\delta_{\Z_\A=\z_\A}(\v)$ is similarly defined.
\end{proposition}

\begin{remark}
    \autoref{Prop: post-order distribution factorize} demonstrates that, given the order structure, the post-decision distribution $p(\v\given\X=\x,~\Z_\A\altereq\z_\A)$ is identifiable from the observational distribution. Further, $\Pr\{\Y\in\S\given\X=\x,~\Z_\A\altereq\z_\A\}$, which constitutes the objective of the AUF problem \eqref{eq: AUF}, can be derived from observational data. This confirms that the influence of AUF decisions can be identified using the variable order, thereby establishing the sufficiency of the order structure for AUF decision-making.
\end{remark}

With the sufficiency of the order structure for AUF decision-making established, we first study order learning. We present our method for learning the variable order in \autoref{Subsection: order learning}, and then introduce the corresponding order-based decision-making method in \autoref{Subsection: order-based decision-making}.

\subsection{Order Learning via Entropy Maximization}
\label{Subsection: order learning}

Recall the random vector $\V=(V_1,\dots,V_d)$ generated by the model $\mathcal M=(\G, \{f_i\}_i, \{\varepsilon_i\}_i)$. In this part, we present our method for learning a variable order $\pi\sim\G$. The method proceeds iteratively: at each iteration, it identifies a sink node among the remaining variables and then removes it, where a sink node is a node with no children. Repeating this procedure for \(d\) iterations yields an order over all variables, and the overall algorithm can be implemented in time polynomial in $d$. The key challenge is therefore how to reliably identify a sink node in each iteration.

To this end, we design an information-theoretic criterion for identifying sink nodes, whose soundness is guaranteed under \autoref{Assumption: OLEM}.
Notably, this assumption imposes no restrictions on the specific form of structural functions $f_i$ or the type of noise distributions, thus providing flexibility across a wide range of data-generating models.
For example, the assumption is naturally satisfied when noise variables have equal information entropy.

\begin{assumption}
    Given any random vector $\V$ generated by the model $\mathcal M=(\G, \{f_i\}_i, \{\varepsilon_i\}_i)$, for any $\pi\sim\G,~m\in[d]$, and $i,j\in\pi[1:m]$ such that $i$ is a sink in $\G^{\pi[1:m]}$, while $j$ is not, there holds
    \begin{align*}
        I(V_j; \Des_j^{\pi[1:m]}\given\W_j^{\pi[1:m]}) > h(\varepsilon_j) - h(\varepsilon_i),
    \end{align*}
    where $d$ denotes the number of nodes in $\G$, $h(\cdot)$ denotes the information entropy, and $I(\cdot;~\cdot)$ denotes the mutual information. $\Des_j\triangleq\{V_l~|~l~\text{is descendant of}~j~\text{in}~\G\}$ denotes the descendants of $V_j$, and $\W_j\triangleq\{V_1,\dots,V_d\}\setminus(\{V_j\}\cup\Des_j)$ denotes the non-descendants of $V_j$. The superscript $\pi[1:m]$ means taking only the first $m$ indices in $\pi$, \emph{i.e.}, $\G^{\pi[1:m]}$ denotes the induced subgraph of $\G$ over the node set $\pi[1:m]\triangleq\{\pi(i)\given 1\leq i\leq m\}$, and $\V^{\pi[1:m]}\triangleq\{V_j\given j\in\pi[1:m]\},~\Des_j^{\pi[1:m]}\triangleq\Des_j\cap\V^{\pi[1:m]},~\W_j^{\pi[1:m]}\triangleq\W_j\cap\V^{\pi[1:m]}$.
    \label{Assumption: OLEM}
\end{assumption}

\begin{remark}    
    From an information-theoretic view, \autoref{Assumption: OLEM} means that, for each non-sink node $j$, the variable $V_j$ passes to its descendants a significant amount of information exceeding the noise information gap $h(\varepsilon_j) - h(\varepsilon_i)$, disregarding other information sources $\W_j^{\pi[1:m]}$.
    This is sensible and mild in practice, since the mutual information term is non-negative, and thus the condition holds whenever $h(\varepsilon_j) < h(\varepsilon_i)$.
    Additionally, from a communication standpoint, \autoref{Assumption: OLEM} lower bounds the conditional channel capacity between a non-sink variable and its descendants.
    \label{Remark: OLEM assumption}
\end{remark}

Under \autoref{Assumption: OLEM}, we present the following proposition, which establishes a criterion for identifying the sink node and serves as the foundation for our proposed order-learning method.

\begin{proposition}
    Under \autoref{Assumption: OLEM}, for any $\pi\sim\G$ and $m\in[d]$, it holds that
    \begin{align*}
        i^* = \mathop{\arg\max}_{i\in\pi[1:m]}~h(V_i\given \V_{-i}^{\pi[1:m]})~\text{implies}~i^*\text{ is a sink in }\G^{\pi[1:m]},
    \end{align*}
    where $\V_{-i}^{\pi[1:m]}\triangleq\{V_j\given j\in\pi[1:m],~j\neq i\}$.
    \label{Prop: sink discrimination iteratively}
\end{proposition}

\begin{remark}
    \autoref{Prop: sink discrimination iteratively} implies that the sink node can be identified via a procedure of conditional entropy maximization. Intuitively, for a sink node $i$, the conditional entropy reduces to the noise entropy, \emph{i.e.}, $h(V_i \given \V_{-i}^{\pi[1:m]}) = h(\varepsilon_i)$. By contrast, for a non-sink node $j$, conditioning on its non-descendants already yields $h(V_j \given \W_j^{\pi[1:m]}) = h(\varepsilon_j)$; conditioning on the remaining variables further reduces this quantity. Provided this reduction exceeds the noise entropy gap $h(\varepsilon_j) - h(\varepsilon_i)$, it follows that $h(V_j \given \V_{-j}^{\pi[1:m]}) < h(V_i \given \V_{-i}^{\pi[1:m]})$, thereby justifying the identification of the sink node by the process of conditional entropy maximization.
\end{remark}

Leveraging \autoref{Prop: sink discrimination iteratively}, we propose the order-learning strategy outlined in \autoref{Alg: OLEM}. This method maintains a set of remaining nodes $\bLambda$, and iteratively identifies the sink node $i^*$ by:
\begin{align*}
    i^*=\mathop{\arg\max}_{i\in\bLambda}~h(V_i \given \V_{-i}^\bLambda),
\end{align*}
where $\V_{-i}^\bLambda \triangleq \{V_j \given j\in\bLambda,~j\neq i\}$. Once identified, $i^*$ is removed from $\bLambda$, and then the process repeats, until there are no remaining nodes, \emph{i.e.} $\bLambda=\emptyset$.

\autoref{Alg: OLEM} operationalizes the information-theoretic principles of sink-node identification established in \autoref{Prop: sink discrimination iteratively}. Additionally, explicit knowledge of the joint distribution $p$ is not required, as \autoref{Alg: OLEM} relies solely on computing the conditional entropy terms.

\setlength{\intextsep}{1pt}
\begin{wrapfigure}[11]{r}{0.55\textwidth}
    \begin{minipage}{\linewidth}
        \begin{algorithm}[H]
        \captionsetup{type=algorithm}
        \caption{OLEM (Order Learning via conditional Entropy Maximization)}
        \label{Alg: OLEM}
        \begin{algorithmic}[1]
        \STATE {\bfseries Input:} Joint distribution $p$ with dimensionality $d$
        \STATE {\bfseries Output:} Order $\pi$
        \STATE Initialize $\bLambda\leftarrow\{1,2,\dots,d\},~\pi=(~)$
        \FOR{$t=1$ {\bfseries to} $d$}
            \STATE Compute $h(\V_{-i}^\bLambda)$ from $p$ for each $i\in\bLambda$
            \STATE$i^* \leftarrow \mathop{\arg\max}_{i\in\bLambda}~h(V_i\given\V_{-i}^\bLambda)$
            \STATE$\pi\leftarrow(i^*, \pi)$ and $\bLambda\leftarrow\bLambda\setminus\{i^*\}$
        \ENDFOR
        \end{algorithmic}
        \end{algorithm}
    \end{minipage}
\end{wrapfigure}

In practice, \autoref{Alg: OLEM} is straightforward to implement, since each iteration requires only a conditional entropy maximization step. Furthermore, the method is compatible with various entropy estimators \citep{entropy_estimate_kl(1nn), entropy_estimate_ensemble, entropy_estimate_knn}. For efficient implementation, we observe that $h(V_i \given \V_{-i}^\bLambda) = h(\V^\bLambda) - h(\V_{-i}^\bLambda)$. Since the joint entropy $h(\V^\bLambda)$ is constant with respect to $i$, the optimization problem simplifies to computing the quantity $h(\V_{-i}^\bLambda)$ and then $i^*=\mathop{\arg\min}_{i\in\bLambda}~h(\V_{-i}^\bLambda)$.

We formally establish the soundness guarantee that \autoref{Alg: OLEM} correctly learns a true order $\pi\sim\G$ under \autoref{Assumption: OLEM}, which is stated in \autoref{Theorem: OLEM}.
\begin{theorem}
    Let $\mathcal{M}=(\G, \{f_i\}_i, \{\varepsilon_i\}_i)$ be a model satisfying \autoref{Assumption: OLEM}, and let $\V$ denote the random vector generated by $\mathcal{M}$. Given the joint distribution $p$ of $\V$ as input, \autoref{Alg: OLEM} returns an order $\pi\sim\G$.
    \label{Theorem: OLEM}
\end{theorem}

Moreover, we position OLEM within a broader landscape by comparing its theoretical scope against relevant algorithms. As summarized in \autoref{Table: comparison to LISTEN, CaPS}, existing algorithms often rely on specific constraints regarding the functional form (e.g., linear) or the noise distribution type (e.g., Gaussian). In contrast, OLEM achieves soundness guarantees across all listed categories without assuming a specific form of structural functions or the type of noise distributions.

\begin{wrapfigure}[9]{r}{0.58\textwidth}
    \vspace{-\intextsep}
    \begin{minipage}{\linewidth}
        \captionsetup{type=table}
        \caption{Availability of theoretical guarantees for the listed algorithms under different settings. ``L'' denotes ``Linear'', ``G'' denotes ``Gaussian'', and ``n'' denotes ``non-''.}
        \label{Table: comparison to LISTEN, CaPS}
        \begin{small}
        \resizebox{\linewidth}{!}{
        \begin{tabular}{c cccccc}
        \toprule
        & CAM & LISTEN & SCORE & DiffAN & CaPS & OLEM \\
        \midrule
        L \& G & No & Yes & No & No & Yes & Yes\\
        L \& nG & No & Yes & No & No & No & Yes\\
        nL \& G & Yes & No & Yes & Yes & Yes & Yes\\
        nL \& nG & No & No & No & Yes & No & Yes\\
        \bottomrule
        \end{tabular}
        }
        \end{small}
    \end{minipage}
    \vspace{-\intextsep}
\end{wrapfigure}

A natural question arises: is this broad adaptivity achieved at the expense of requiring stricter assumptions under simpler settings? For instance, in the linear Gaussian case, does OLEM impose significantly stronger assumptions than specialized methods like LISTEN \citep{aistats/LISTEN2018} and CaPS \citep{nips/CaPS2024}?
To facilitate the comparison, we review the assumptions of LISTEN and CaPS.

\begin{assumption}[LISTEN \citep{aistats/LISTEN2018}]
    Given any random vector $\V$ generated by the model $\mathcal M=(\G, \{f_i\}_i, \{\varepsilon_i\}_i)$ under the linear setting:
    \begin{align*}
        f_i(\Pa_i)=\sum_{V_j\in\Pa_i}w_{ji}V_j,
    \end{align*}
    for any $\pi\sim\G,~m\in[d]$, and for any $i,j\in\pi[1:m]$, if $i$ is a sink in $\G^{\pi[1:m]}$, while $j$ is not, then
    \begin{align*}
        \sigma_i^{-2} < \sigma_j^{-2} + \sum_{l: V_j\in\Pa_l^{\pi[1:m]}}\sigma_l^{-2}w_{jl}^2,
    \end{align*}
    where $d$ denotes the number of nodes in $\G$, and $\sigma_i^2$ is variance of $\varepsilon_i$ with $\E[\varepsilon_i]=0$ for any $i\in[d]$.
    \label{Assumption: LISTEN}
\end{assumption}

\begin{assumption}[CaPS \citep{nips/CaPS2024}]
    \label{Assumption: CaPS}
    Given any random vector $\V$ generated by the model $\mathcal M=(\G, \{f_i\}_i, \{\varepsilon_i\}_i)$ under the Gaussian setting: $\varepsilon_i\sim\mathcal N(0,\sigma_i^2)$, for any $\pi\sim\G$, denoting by $d$ the number of nodes in $\G$, at least one of the following holds.
    \begin{enumerate}[(1)]
        \item For any $i,j\in[d]$, $\pi^{-1}(i) < \pi^{-1}(j)$ implies $\sigma_i < \sigma_j$.
        \item For any $m\in[d]$, and for any $j\in\pi[1:m]$, if $j$ is non-sink in $\G^{\pi[1:m]}$, then
        \begin{align*}
            \sigma_{\min}^{-2} < \sigma_j^{-2} + \sum_{l: V_j\in\Pa_l^{\pi[1:m]}}\sigma_l^{-2}\E[(\frac{\partial f_l(\pa_l)}{\partial v_j})^2].
        \end{align*}
    \end{enumerate}
\end{assumption}

We present \autoref{Prop: Assumption coincide} to demonstrate that, when restricted to linear Gaussian cases, the assumption of OLEM reduces to that of LISTEN and is no stronger than that of CaPS. This confirms the adaptivity of OLEM does not come at the cost of stricter assumptions under settings such as linear Gaussian.

\begin{proposition}
    Under the linear and Gaussian setting, \autoref{Assumption: OLEM} reduces to \autoref{Assumption: LISTEN}, and \autoref{Assumption: CaPS} implies \autoref{Assumption: OLEM}.
    \label{Prop: Assumption coincide}
\end{proposition}

\subsection{Order-based Decision-Making}
\label{Subsection: order-based decision-making}

In this part, we propose an order-based decision-making method for addressing the AUF problem \eqref{eq: AUF}. At first, we leverage conditional flow models \citep{conditional_flows_ardizzone2019, conditional_flows_winkler2019} to construct a sampler for the post-decision distribution $p_{\Y\given\X=\x,~\Z_\A\altereq\z_\A}(\cdot\given\X=\x,~\Z_\A\altereq\z_\A)$ based on the learned order. The sampler further enables approximating the objective function in \eqref{eq: AUF}.

We here detail the construction of the sampler. As presented in \autoref{Prop: post-order distribution factorize}, given the order, the joint post-decision distribution $p_{\V\given\X=\x,~\Z_\A\altereq\z_\A}$ is identifiable from the observational distribution. Further, direct estimation of the marginal distribution $p_{\Y\given\X=\x,~\Z_\A\altereq\z_\A}$ typically involves an intractable integral form. To avoid this, we conduct a two-step strategy: (1) we construct a sampler for the joint distribution $p_{\V\given\X=\x,~\Z_\A\altereq\z_\A}(\cdot\given\X=\x,~\Z_\A\altereq\z_\A)$ and (2) we extract the components corresponding to $\Y$ to obtain a sampler for $p_{\Y\given\X=\x,~\Z_\A\altereq\z_\A}(\cdot\given\X=\x,~\Z_\A\altereq\z_\A)$.

\setlength{\intextsep}{1pt}
\begin{wrapfigure}[17]{r}{0.58\textwidth}
    \begin{minipage}{\linewidth}
        \begin{algorithm}[H]
        \captionsetup{type=algorithm}
   \caption{Construction of the sampler for the joint post-decision distribution}
   \label{Alg: sampler_construction}
\begin{algorithmic}[1]
    \STATE {\bfseries Input:} Order $\pi$, conditional flow models $\{T_i\}_{i=1}^d$, noises $\{N_i\}_{i=1}^d$, context $\x$, decision $\z_\A$
    \STATE {\bfseries Output:} Constructed sampler $T$
    \STATE Initialize $T$ as an empty random vector $(~)$
    \FOR{$i=\pi(1)$ {\bfseries to} $\pi(d)$}
        \IF{$V_i$ is contextual, \emph{i.e.} $\exists j,~V_i=X_j$}
        \STATE $T := (T, x_j)$
        \ELSIF{$V_i$ is alterable, \emph{i.e.} $\exists j,~V_i=Z_{a_j}$}
        \STATE $T := (T, z_{a_j})$
        \ELSE
        \STATE $T := (T, T_i(N_i,T))$
        \ENDIF
    \ENDFOR
\end{algorithmic}
        \end{algorithm}
    \end{minipage}
\end{wrapfigure}

The core challenge lies in step (1), for which we propose the following implementation. Given an order $\pi$, we associate each variable $V_i$ with a conditional flow model $T_i(N_i, \pred_i(\pi))$ \citep{conditional_flows_winkler2019}, which transforms a standard Gaussian noise $N_i \sim \mathcal{N}(0,1)$ and the values of the predecessors $\pred_i(\pi)$ into a variable following the conditional distribution of $V_i$, \emph{i.e.} $T_i(N_i, \pred_i(\pi)) \sim p_{V_i\given \pred_i(\pi)}(\cdot\given \pred_i(\pi))$. Subsequently, we employ \autoref{Alg: sampler_construction} to construct a joint sampler $T(\N,\x,\z_\A)$ by composing the individual $T_i$-s according to $\pi$, where $\N\triangleq(N_i)_i$ denotes the joint noise vector. Based on \autoref{Prop: post-order distribution factorize} and the aforementioned property of $T_i$-s, this construction ensures that $T(\N,\x,\z_\A)\sim p_{\V\given\X=\x,~\Z_\A\altereq\z_\A}(\cdot\given\X=\x,~\Z_\A\altereq\z_\A)$ when $\N\sim\mathcal N(\mathbf 0, \mathbf I)$. Notably, the structure of the constructed sampler remains valid for varying contexts $\x$ and decisions $\z_\A$ without the need for reconstruction.

Completing the procedure with step (2), we extract the components of the joint sampler $T(\N,\x,\z_{\A})$ corresponding to the target variables $\Y$. This yields the marginal sampler $T_{\Y}(\N,\x,\z_{\A})\sim p_{\Y\given\X=\x,~\Z_\A\altereq\z_\A}(\cdot\given\X=\x,~\Z_\A\altereq\z_\A)$, where $\N\sim\mathcal N(\mathbf 0, \mathbf I)$.

To learn the conditional flow models $T_i$-s, various methods are available. In this paper we employ the conditional invertible neural network (cINN) \citep{conditional_flows_ardizzone2019} to fit each $T_i$ by minimizing the empirical KL divergence between the output and the target distribution.

Leveraging the sampler, we can approximate the objective of problem \eqref{eq: AUF} and subsequently solve the problem. Specifically, the sampler allows us to express the objective as $\Pr\{T_{\Y}(\N,\x,\z_{\A})\in\S\}$. Consequently, for any candidate decision $\z_{\A}$, we can estimate the objective by sampling, \emph{i.e.}
\begin{align*}
    &\Pr\{\Y\in\S\given\X=\x,~\Z_{\A}\altereq\z_{\A}\}\\
    =&~\Pr\{T_{\Y}(\N,\x,\z_{\A})\in\S\}\\
    \approx&~\frac1n\sum_{i=1}^n \I[T_{\Y}(\N^i,\x,\z_{\A})\in\S],
\end{align*} 
where $\{\N^i\}_{i=1}^n$ are i.i.d. samples of $\N$, and $\I[\cdot]$ denotes the indicator function. 

This estimation allows for solving problem \eqref{eq: AUF} using zeroth-order optimization algorithms, such as Bayesian Optimization \citep{icml/BO_constrained2014}, yet these methods are often computationally inefficient. To address this, we adopt the surrogate objective introduced by \citet{aaai/Qin2025gradient}. In this manner, we convert the problem into a differentiable optimization task, which can be solved via efficient gradient-based methods. The converted problem is given by the following:
\begin{align}
\begin{split}
    \min_{\z_{\A}}\quad &\frac1n\sum_{i=1}^n\Vert T_{\Y}(\N^i,\x,\z_{\A})-\c\Vert_1\\
    \text{s.t.}\quad &~\z_{\A}\in \mathbf\Delta_{\A},
    \label{eq: AUF_reduced}
\end{split}
\end{align}
where $\c$ denotes the Chebyshev center of $\S$, and is computable by linear programming \citep{convex_opt_boyd2004}. Since $T_{\Y}$ is instantiated using differentiable neural networks, the objective in \eqref{eq: AUF_reduced} is fully differentiable with respect to $\z_{\A}$. This allows us to solve the problem using gradient-based optimizers; in our implementation, we employ the Adam optimizer \citep{iclr/adam2015}.

\setlength{\intextsep}{1pt}
\begin{wrapfigure}[17]{r}{0.62\textwidth}
    \begin{minipage}{\linewidth}
    \begin{algorithm}[H]
    \captionsetup{type=algorithm}
    \caption{OLEM-Rh}
    \label{Alg: OLEM-Rh}
        \begin{algorithmic}[1]
        \STATE {\bfseries Input:} Desired region $\S$ parameterized by $\M,\d$, context $\x$, alterable variable indices $\A$, domain $\mathbf\Delta_\A$, observational data $\D=\{\V^i\}_i$
        \STATE {\bfseries Output:} Decision $\z_{\A}^\star$
        \STATE Learn an order $\pi$ via \autoref{Alg: OLEM}, where conditional entropy terms are estimated with $\D$
        \FOR{$i=1$ {\bfseries to} $d$}
            \STATE Fit the conditional flow model $T_i$ using a cINN trained with $\pi$ and $\D$
        \ENDFOR
        \STATE Construct $T$ via \autoref{Alg: sampler_construction}, and further obtain $T_{\Y}$
        \STATE Compute $\c$ from $\M,\d$, and sample $\{\N^i\}_{i=1}^n\stackrel{\text{i.i.d.}}{\sim}\mathcal N(\mathbf 0, \mathbf I)$
        \STATE Obtain $\z_{\A}^\star$ by solving \eqref{eq: AUF_reduced} using the Adam optimizer with $T_\Y$, $\c$, $\{\N^i\}_{i=1}^n$, $\x$, $\A$, $\mathbf\Delta_\A$
        \end{algorithmic}
    \end{algorithm}
    \end{minipage}
    \vspace{-\intextsep}
\end{wrapfigure}

By integrating the proposed decision-making procedure with the order learned via \autoref{Alg: OLEM}, we establish the method OLEM-Rh, detailed in \autoref{Alg: OLEM-Rh}. This method leverages observational data to learn the variable order and subsequently optimize decisions for the AUF problem. A key strength of OLEM-Rh is its modular architecture, which allows several components to be substituted with alternative techniques, providing strong flexibility for the method. For instance, the specific gradient-based optimizer in line 9 can be varied. More fundamentally, the optimization strategy (lines 8-9) can be replaced by directly solving the primal AUF problem (\ref{eq: AUF}) using zeroth-order (black-box) methods, such as Bayesian Optimization \citep{icml/BO_constrained2014}. Similarly, the conditional flow models $T_i$-s (line 5) can be instantiated using other generative models, provided that they satisfy the differentiability requirements of the chosen optimization method.

\textbf{Discussion.} The order structure leveraged in this work aligns with the concept of causal order. To the best of our knowledge, the concept of causal ordering traces back to the 1950s, most notably in the seminal work of \citet{causal_order1953}.
Simon formalized the asymmetric relations in econometrics, distinguishing directional influence from symmetric algebraic equations. This foundational framework later transcended social sciences, inspiring research in qualitative physics~\citep{IWASAKI198663} and artificial intelligence~\citep{causal_book_Pearl2000}. In recent years, causal ordering has seen a resurgence, primarily serving as an intermediate representation for learning causal graphs~\citep{uai/Teyssier2005, CAM2014, aistats/LISTEN2018, icml/SCORE2022, iclr/DiffAN2023, nips/CaPS2024}. In this work, however, we argue that the order structure itself holds intrinsic value for decision-making, specifically for AUF. We show that the topological properties of causal order are directly beneficial for these problems and propose the first order-based method tailored for AUF.

\section{Experiments}

In this section, we conduct experiments to validate the effectiveness of the proposed methods. In \autoref{Subsection: expt_order}, we evaluate the performance of OLEM for learning the order structure. In \autoref{Subsection: expt_decision}, we then assess OLEM-Rh on AUF decision-making tasks.

\subsection{Order Learning}
\label{Subsection: expt_order}

We compare OLEM with five representative baselines, including LISTEN \citep{aistats/LISTEN2018} (primarily designed for linear settings), DiffAN \citep{iclr/DiffAN2023} (for nonlinear settings), CAM \citep{CAM2014} and SCORE \citep{icml/SCORE2022} (tailored for nonlinear Gaussian cases), and CaPS \citep{nips/CaPS2024} (for linear or nonlinear Gaussian cases).

We assess the quality of the learned order using the order divergence (DIV) \citep{icml/SCORE2022}. Since most baselines return a graph besides an order, we additionally report the Structural Hamming Distance (SHD) and the Structural Intervention Distance (SID) \citep{sid_shd_buhlmann2015} of the returned graphs. To compute graph-based metrics for OLEM, we first convert the learned order into a DAG by applying the CAM pruning procedure \citep{CAM2014} with cutoff $0.001$, and then evaluate SHD and SID on the resulting graph.

We conduct experiments on both synthetic and real-world data. For synthetic evaluation, we generate datasets by varying the dimensionality $d$, edge density $p$, linearity $r$, and noise type. Specifically, we set $d \in \{5,10,20\}$ and generate true graphs using the Erd{\"o}s-R{\'e}nyi model \citep{ERgraph} with edge density $p \in \{0.3,0.5,0.8\}$. The linearity of each structural function is controlled by $r \in \{0,0.25,0.5,0.75,1\}$: each $f_i$ is linear with probability $r$ and nonlinear otherwise. Linear coefficients are sampled uniformly from $[-1,-0.25]\cup[0.25,1]$ to keep them bounded away from zero, while nonlinear functions are generated by sampling from Gaussian processes with a unit bandwidth RBF kernel. Noise terms are drawn from beta, exponential, or Gaussian distributions, with distribution parameters sampled uniformly from $[0.75,1.25]$. The observational sample size is $2000$. To reduce instance-specific variance, we adopt a hierarchical evaluation protocol: For each synthetic setting, we report the mean and standard deviation over 5 independent epochs; within each epoch, we repeatedly generate tasks with different ground-truth models and average the metrics.

For real-world evaluation, we use the Sachs dataset \citep{sachs_dataset2005}, a protein signaling network with $11$ nodes and $17$ edges, containing $853$ samples. We randomly shuffle the dataset in each run and report the mean and standard deviation of each metric over $10$ independent runs.

\begin{figure*}[t]
\centering
\begin{subfigure}[]{0.49\linewidth}
    \includegraphics[width=\linewidth]{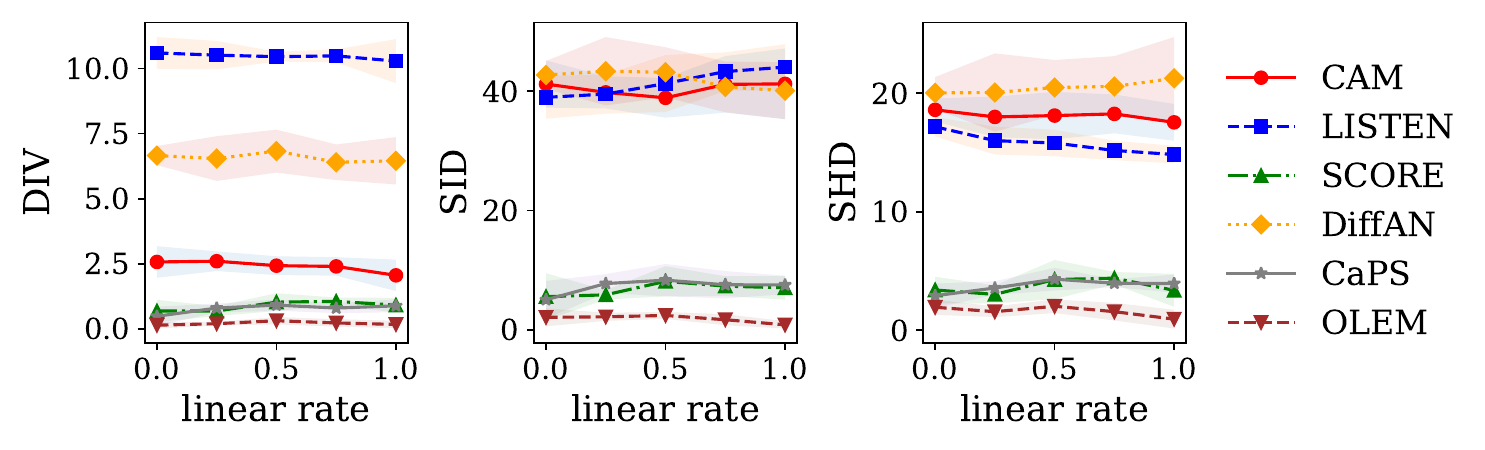}
    \caption{$d=10,~p=0.3$ with beta noise}
    \label{Fig: beta_10_0d3_inmainbody}
\end{subfigure}
\hfill
\begin{subfigure}[]{0.49\linewidth}
    \includegraphics[width=\linewidth]{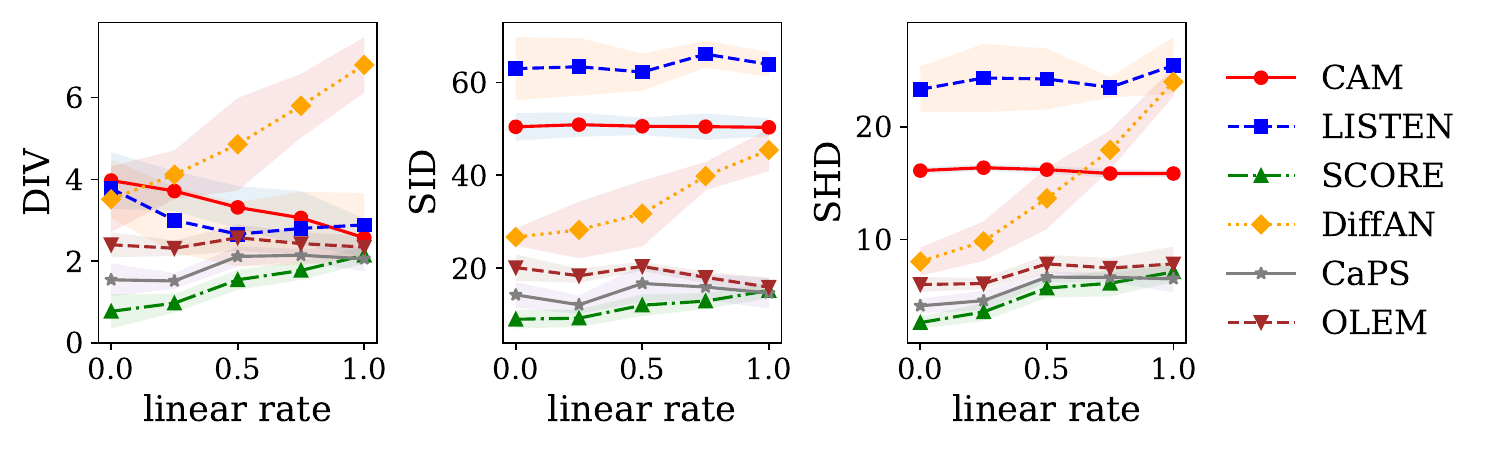}
    \caption{$d=10,~p=0.3$ with Gaussian noise}
    \label{Fig: gauss_10_0d3_inmainbody}
\end{subfigure}
\caption{Results on the synthetic datasets with beta or Gaussian noise ($d=10, p=0.3$)}
\label{Fig: beta_gauss_10_0d3}
\end{figure*}

\autoref{Fig: beta_gauss_10_0d3} reports experimental results on synthetic datasets with beta or Gaussian noise ($d=10, p=0.3$). Under the Gaussian noise setting, CaPS and SCORE achieve the strongest performance, which aligns with their specified Gaussian noise assumptions.

\begin{wrapfigure}[13]{r}{0.45\textwidth}
    \begin{minipage}{\linewidth}
        \captionsetup{type=table}
        \caption{Results on the real dataset Sachs. For each metric, the best results are bolded.}
        \label{Table: causal experiment real data}
        \begin{tabular}{l rrr} 
        \toprule
        Dataset & \multicolumn{3}{c}{Sachs}\\
        \midrule
        Method & \makecell[c]{DIV$\downarrow$} & \makecell[c]{SHD$\downarrow$} & \makecell[c]{SID$\downarrow$}\\
        \midrule
        CAM & 8.0\textpm0.0 & 17.0\textpm0.0 & \myfontb{45.0\textpm0.0}\\
        LISTEN & 8.0\textpm0.0 & 30.8\textpm6.7 & 54.1\textpm4.1\\
        SCORE & 8.0\textpm0.0 & 17.0\textpm0.0 & \myfontb{45.0\textpm0.0}\\
        DiffAN & \myfontb{7.2\textpm1.9} & 17.8\textpm2.9 & 49.7\textpm4.4\\
        CaPS & 9.0\textpm0.0 & \myfontb{16.0\textpm0.0} & 46.0\textpm0.0\\
        OLEM & 8.0\textpm0.0 & \myfontb{16.0\textpm0.0} & \myfontb{45.0\textpm0.0}\\
        \bottomrule
        \end{tabular}
    \end{minipage}
\end{wrapfigure}

Under the non-Gaussian noise setting, OLEM consistently performs better than these Gaussian-oriented methods and shows a clear margin over the remaining baselines. This demonstrates that OLEM adapts well to diverse data-generating distributions rather than benefiting from a specific noise model. Results for additional dimensionalities, edge densities, and noise types are provided in the appendix.

Table \ref{Table: causal experiment real data} summarizes the results on the Sachs dataset. OLEM achieves the best SHD and SID, while ranking second to DiffAN in terms of DIV. Notably, although DiffAN attains the best mean DIV, it also exhibits the worst mean SID and large variance across data shuffles. This suggests that DiffAN may capture certain ordering-related statistics, whereas OLEM yields a more robust and reliable recovery of the underlying structure.

\subsection{AUF Decision-Making}
\label{Subsection: expt_decision}

We now evaluate OLEM-Rh on the AUF decision-making problem and focus on three questions:
\begin{enumerate}
    \item [(1)] Whether AUF decision-making benefits more from learning an explicit graph and using Grad-Rh \citep{aaai/Qin2025gradient} or from learning an order and using OLEM-Rh, \emph{i.e.}, whether an explicit graph is necessary for AUF decision-making?
    \item [(2)] Whether OLEM is a crucial component for the order-learning stage or OLEM-Rh remains competitive when coupled with alternative order-learning methods?
    \item [(3)] How close OLEM-Rh can get to oracle decision-makers such as Grad-Rh and QWZ23 \citep{nips/QWZ2023} that assume access to the true graph?
\end{enumerate}

Decision quality is measured by the success probability of avoiding the undesired future. Specifically, for each recommended decision, we simulate the decision on the ground-truth model for $1000$ independent trials, and measure the resulting empirical success rate as the metric.

We conduct experiments on both synthetic and real-world data. For synthetic evaluation, we consider both linear and nonlinear data-generating processes with beta and exponential noise. For each synthetic setting, we construct $20$ tasks with distinct true models, alterable variable sets, decision domains, and desired regions. We compute for each task the average success probability over $50$ decision rounds, and then report the mean and standard deviation across the $20$ tasks.

The true graphs are generated using the Erd{\"o}s-R{\'e}nyi model \citep{ERgraph} with edge probability $p=0.3$. In linear settings, edge weights are sampled uniformly from $[-1,1]$. In nonlinear settings, structural functions are instantiated using three-layer MLPs with sigmoid activations, where weights are sampled uniformly from $[-2,-0.5]\cup[0.5,2]$. Noise terms are generated following \autoref{Subsection: expt_order}.
The set of alterable variables is the union of randomly selected ancestors for each $Y_i$ (where each ancestor is selected with probability 0.5).
For each alterable variable $Z_{a_i}$, the decision domain is $[\mu-2\sigma,\mu+2\sigma]$, and for each outcome $Y_i$, the desired region is $[\mu-\sigma,\mu+\sigma]$, where $\mu$ and $\sigma$ denote the empirical mean and standard deviation of the corresponding variable, respectively. We set the dimensionality to $15$ in linear settings and $12$ in nonlinear settings.

For real-world evaluation, we use the Bermuda dataset \citep{bermuda2017, bermuda2018}, which contains environmental measurements from the Bermuda region. Following \citep{nips/QWZ2023}, we employ their fitted linear Gaussian simulators to generate $10$ pseudo-real datasets. We then follow the same evaluation protocol as in the synthetic experiments: For each dataset, we run $50$ independent decision rounds and compute the average success probability; we then report the mean and standard deviation across the $10$ datasets.

\begin{table*}[t]
\caption{Success probability on four synthetic settings and Bermuda. Methods in the first part learn the order or graph from data, while methods in the second part are provided with the true graph. Within each part, results with the highest mean are bolded. The sample size is $1000$.}
\label{Table: rehearsal experiment (est. graph)}
\begin{center}
\begin{tabular}{l ccccc}
\toprule
 & SynLrBeta & SynLrExp & SynNlrBeta & SynNlrExp & Bermuda\\
\midrule
OLEM-Rh & \myfontb{.772\textpm.130} & \myfontb{.785\textpm.089} & \myfontb{.820\textpm.173} & \myfontb{.844\textpm.085} & \myfontb{.668\textpm.006}\\
Grad-Rh w/OLEM+prun & .738\textpm.196 &.761\textpm.193 &.816\textpm.155 & .840\textpm.082 & \myfontb{.668\textpm.006} \\
Grad-Rh w/CaPS & .712\textpm.210 &.705\textpm.220 &.795\textpm.213 & .839\textpm.089 & .575\textpm.107 \\
Grad-Rh w/NOTEARS & .453\textpm.264 &.396\textpm.247 &.795\textpm.165 & .831\textpm.078 & .322\textpm.021 \\
Grad-Rh w/GES & .654\textpm.223 &.590\textpm.270 &.786\textpm.222 & .829\textpm.081 & .279\textpm.144 \\
OLEM-Rh w/CaPS & .752\textpm.148 &.750\textpm.121 &.802\textpm.180 & .808\textpm.126 & \myfontb{.668\textpm.005} \\
\midrule
Grad-Rh (true graph) & \myfontb{.793\textpm.113} & \myfontb{.828\textpm.075} & \myfontb{.843\textpm.176} & \myfontb{.851\textpm.071} & \myfontb{.669\textpm.007} \\
QWZ23 (true graph) & .754\textpm.096 & .768\textpm.205 & .640\textpm.284 & .804\textpm.201 & \myfontb{.669\textpm.008} \\
\bottomrule
\end{tabular}
\end{center}
\vskip -0.16in
\end{table*}

To assess whether an explicit graph is necessary for AUF decisions, we compare OLEM-Rh to pipelines that first learn a graph from observational data and then apply Grad-Rh, which was originally formulated under access to the true graph \citep{aaai/Qin2025gradient}. Concretely, we pair Grad-Rh with several structure learning algorithms, including OLEM+prun (where we apply CAM pruning \citep{CAM2014} with cutoff 0.001 to the order learned by OLEM to obtain a graph), CaPS \citep{nips/CaPS2024}, NOTEARS \citep{nips/NOTEARS2018}, and GES \citep{GES2002}. As shown in \autoref{Table: rehearsal experiment (est. graph)}, across all the five experimental settings, OLEM-Rh consistently outperforms Grad-Rh equipped with a learned graph. This supports the intuition that learning the graph structure introduces edge-level errors that can be harmful to downstream decision-making, and that an explicit graph can be unnecessary for strong AUF decision-making performance.

We next examine whether OLEM is a key ingredient in the rehearsal framework by replacing the order-learning component with alternative methods (e.g., CaPS). As shown in \autoref{Table: rehearsal experiment (est. graph)}, OLEM-Rh with OLEM achieves higher success probabilities than the same rehearsal procedure coupled with alternative order learners. This indicates that OLEM is particularly well-suited for the rehearsal framework, likely due to its robustness under non-Gaussian settings.

Finally, we compare OLEM-Rh with oracle methods that have access to the true graph, namely Grad-Rh and QWZ23 \citep{nips/QWZ2023}. In these comparisons, the true graph is provided to Grad-Rh and QWZ23 but remains unknown to OLEM-Rh. As expected, the oracle method Grad-Rh provides an empirical upper bound. Nevertheless, OLEM-Rh achieves performance comparable to Grad-Rh across all test cases, and outperforms QWZ23 on the four synthetic settings, where the superiority of OLEM-Rh to QWZ23 is due to the rigid linear Gaussian assumption of QWZ23 (which limits its efficacy in non-Gaussian environments, while ensures its performance on the linear Gaussian Bermuda dataset).

\begin{wrapfigure}{r}{0.5\textwidth}
    \vspace{-\intextsep}
    \begin{minipage}{0.5\textwidth}
        \captionsetup{type=figure}
        \begin{subfigure}[]{0.49\linewidth}
            \includegraphics[width=\linewidth]{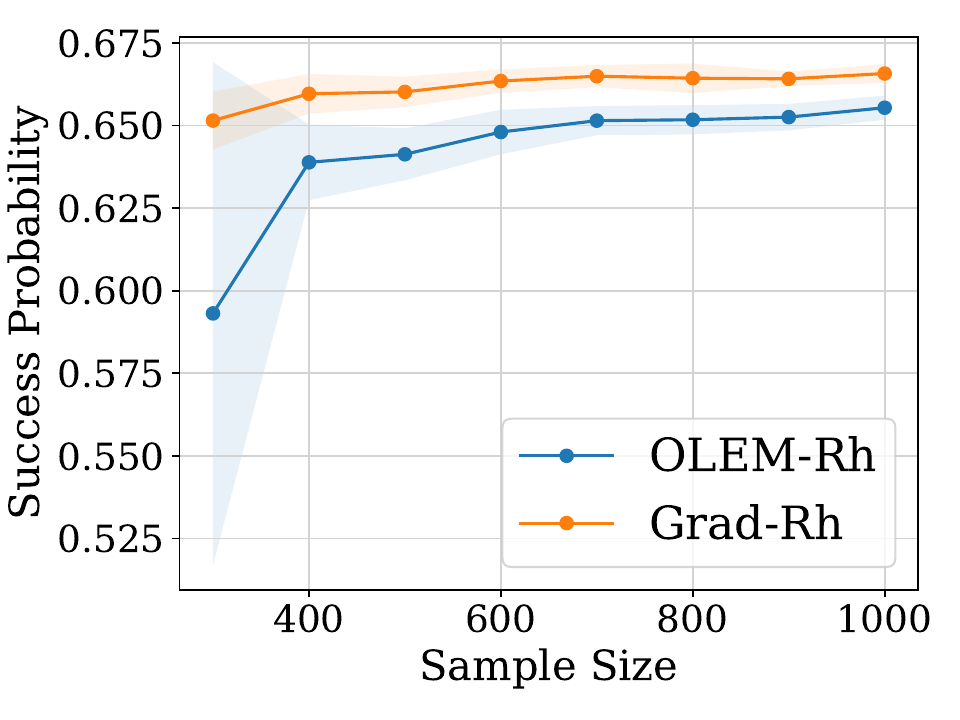}
            \caption{}
            \label{Fig: performance_compare}
        \end{subfigure}
        \begin{subfigure}[]{0.49\linewidth}
            \includegraphics[width=\linewidth]{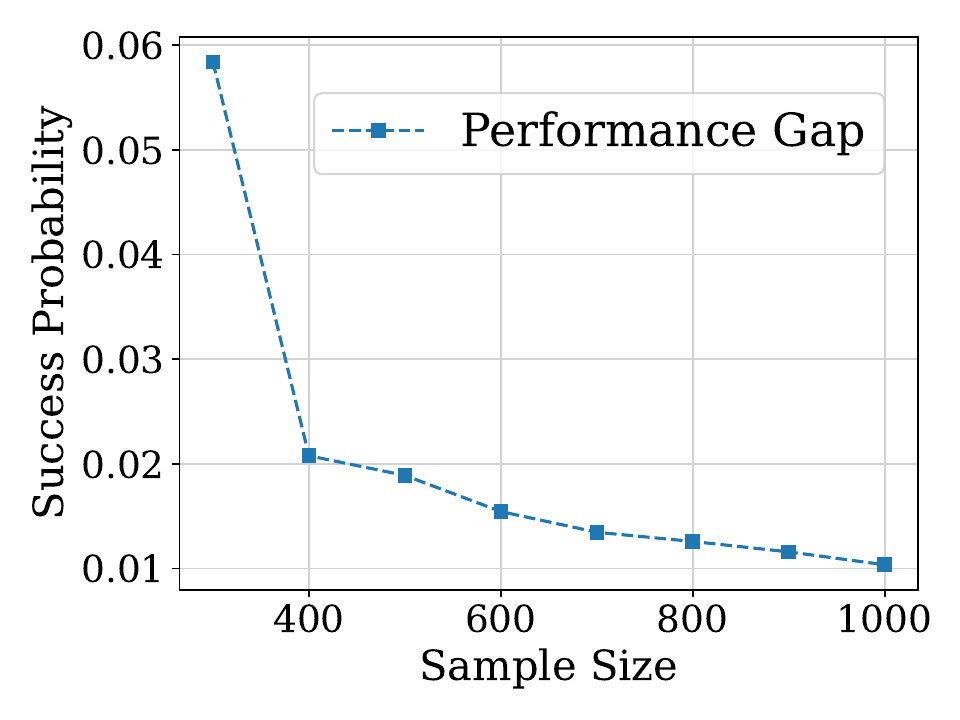}
            \caption{}
            \label{Fig: performance_gap}
        \end{subfigure}
        \caption{(a) Relation between the success probability and the sample size on the Bermuda dataset. (b) Relation between the performance gap and the sample size on the Bermuda dataset.}
        \label{Fig: expt_n_samples}
    \end{minipage}
    \vspace{-\intextsep}
\end{wrapfigure}

We further study the effect of sample size on Bermuda. \autoref{Fig: performance_compare} shows absolute performance as a function of sample size, and \autoref{Fig: performance_gap} reports the gap $p_{\text{Grad-Rh}} - p_{\text{OLEM-Rh}}$. Both methods improve as sample size increases, reflecting more accurate estimation of variable relations. Moreover, the gap narrows with more data: the initial rapid decrease suggests that OLEM-Rh quickly identifies a high-quality order, while the subsequent slower decay indicates that, once the order stabilizes, performance converges as the conditional flow models become the dominant factor in capturing the underlying relations, diminishing the marginal advantage of the graph structure.

\section{Conclusion}

In this paper, we demonstrate that the graph structure can be unnecessary for decision-making in AUF problems, and establish the sufficiency of the variable order for identifying the influence of decisions from observational data.
We propose the first order-based rehearsal learning method OLEM-Rh to learn and exploit the order structure for AUF decision-making. Notably, several modules of this method can be substituted with alternative techniques, providing strong flexibility. Experimental results show that OLEM-Rh performs comparably to oracle baselines that are provided with access to the true graph, and outperforms pipelines that require learning a graph, validating the sufficiency of the order structure for AUF decision-making.

\bibliographystyle{plainnat}
\bibliography{main.bib}

\newpage
\appendix

\section{Proofs}

\label{Appendix: proofs}

\subsection{Proof of \autoref{Prop: post-order distribution factorize}}

\theoremstyle{plain}
\newtheorem*{proposition31}{\textbf{Proposition 3.1}}
\begin{proposition31}
    Given a random vector $\V$ generated by the model $\mathcal M=(\G, \{f_i\}_i, \{\varepsilon_i\}_i)$, suppose $\V$ is separated into $\X,\Z,\Y$ with $\Z_\A\subseteq\Z$ according to the AUF setting described in $\autoref{Section: AUF formulation}$. 
    Given an order $\pi\sim\G$, it holds that
    \begin{align*}
        p(\v\given\X=\x,~\Z_\A\altereq\z_\A) = \delta_{\X=\x}(\v)~\delta_{\Z_\A=\z_\A}(\v)\prod_{i}p(v_i\given \pred_i(\pi)).
    \end{align*}
    where the product is taken over the $V_i$-s that are not contained in $\X$ or $\Z_\A$. $\Pred_i(\pi)\triangleq\{V_j\given \pi^{-1}(j)<\pi^{-1}(i)\}$ denotes the predecessors of $V_i$ according to $\pi$, and $\pred_i(\pi)$ denotes an assignment to $\Pred_i(\pi)$. $\delta_{\X=\x}(\v)$ is a Dirac function satisfying: (1) $\delta_{\X=\x}(\v)=0$ if the sub-vector of $\v$ corresponding to $\X$ is not equal to $\x$, and (2) $\int\delta_{\X=\x}(\v)\d\v=1$. The Dirac function $\delta_{\Z_\A=\z_\A}(\v)$ is similarly defined.
\end{proposition31}

\begin{proof}
    Recall from \autoref{Section: AUF formulation} that any variable in $\Z$ or $\Y$ cannot be a parent of a contextual variable $X_i$. This gives
    \begin{align*}
        p(\v\given\X=\x,~\Z_\A\altereq\z_\A) = p(\v\given\X\altereq\x,~\Z_\A\altereq\z_\A).
    \end{align*}
    The model $\mathcal M=(\G, \{f_i\}_i, \{\varepsilon_i\}_i)$ satisfies the Markov property, \emph{i.e.} each $V_i$ is independent of its non-descendants given $\Pa_i$ \citep{causal_book_Pearl2000}. This gives $p(\v)=\prod_{i\in[d]}p(v_i\given\pa_i)$. Further, when altering $\X$ and $\Z_\A$, for any unaltered $V_i$, the relation $V_i=f_i(\Pa_i)+\varepsilon_i$ remains unchanged. Therefore it holds that
    \begin{align*}
        p(\v\given\X\altereq\x,~\Z_\A\altereq\z_\A) = \delta_{\X=\x}(\v)\delta_{\Z_\A=\z_\A}(\v)\prod_{i}p(v_i\given \pa_i),
    \end{align*}
    where the product is taken over the $V_i$-s that are not contained in $\X$ or $\Z_\A$. Also by the Markov property we have $p(v_i\given\pa_i)=p(v_i\given\pred_i(\pi))$, since $\Pred_i(\pi)$ includes all parents of $V_i$ and contains no descendants of $V_i$. Therefore we obtain
    \begin{align*}
        p(\v\given\X\altereq\x,~\Z_\A\altereq\z_\A) = \delta_{\X=\x}(\v)\delta_{\Z_\A=\z_\A}(\v)\prod_{i}p(v_i\given \pred_i(\pi)).
    \end{align*}
    This finishes the proof.
\end{proof}

\subsection{Proof of \autoref{Prop: sink discrimination iteratively}}

\newtheorem*{proposition32}{\textbf{Proposition 3.2}}
\begin{proposition32}
    Under \autoref{Assumption: OLEM}, for any $\pi\sim\G$ and $m\in[d]$, it holds that
    \begin{align*}
        i^* = \mathop{\arg\max}_{i\in\pi[1:m]}~h(V_i\given \V_{-i}^{\pi[1:m]})~\text{implies}~i^*\text{ is a sink in }\G^{\pi[1:m]},
    \end{align*}
    where $\V_{-i}^{\pi[1:m]}\triangleq\{V_j\given j\in\pi[1:m],~j\neq i\}$.
\end{proposition32}

\begin{proof}
    For any sink $i$ in $\G^{\pi[1:m]}$, since $\pi$ is an order of $\G$, it is guaranteed that $\V_{-i}^{\pi[1:m]}$ includes all parents of $V_i$ in $\G$. Also, $V_i$ has no descendants. By the Markov property we have
    \begin{align*}
        h(V_i\given \V_{-i}^{\pi[1:m]}) = h(V_i\given\Pa_i) = h(\varepsilon_i).
    \end{align*}
    For any non-sink $j$ in $\G^{\pi[1:m]}$, it holds that
    \begin{align*}
        h(V_j\given \V_{-j}^{\pi[1:m]}) = h(V_j\given \Des_j^{\pi[1:m]},\W_j^{\pi[1:m]}) = h(V_j\given\W_j^{\pi[1:m]}) - I(V_j;\Des_j^{\pi[1:m]}\given\W_j^{\pi[1:m]}).
    \end{align*}
    Note that $\W_j^{\pi[1:m]}$ includes all parents of $V_j$ in $\G$ and contains no descendants of $V_j$ in $\G$. This gives $h(V_j\given\W_j^{\pi[1:m]}) = h(V_j\given\Pa_j) = h(\varepsilon_j)$, and it holds that
    \begin{align*}
        h(V_i\given \V_{-i}^{\pi[1:m]}) - h(V_j\given \V_{-j}^{\pi[1:m]}) = h(\varepsilon_i) - h(\varepsilon_j) + I(V_j;\Des_j^{\pi[1:m]}\given\W_j^{\pi[1:m]}) > 0.
    \end{align*}
    Therefore, $i^* = \mathop{\arg\max}_{i\in\pi[1:m]}~h(V_i\given \V_{-i}^{\pi[1:m]})$ implies $i^*$ is a sink node in $\G^{\pi[1:m]}$.
\end{proof}

\subsection{Proof of \autoref{Theorem: OLEM}}

\newtheorem*{theorem31}{\textbf{Theorem 3.1}}
\begin{theorem31}
    Let $\mathcal{M}=(\G, \{f_i\}_i, \{\varepsilon_i\}_i)$ be a model satisfying \autoref{Assumption: OLEM}, and let $\V$ denote the random vector generated by $\mathcal{M}$. Given the joint distribution $p$ of $\V$ as input, \autoref{Alg: OLEM} returns an order $\pi\sim\G$.
\end{theorem31}

\begin{proof}
    Denote by $\pi$ the current order and by $\bLambda$ the current remaining nodes maintained by \autoref{Alg: OLEM}. We prove by induction that for any $t\in[d]$, after the $t$-th iteration, $\pi$ is the tail of some order of $\G$ (that is, $\pi(1),\dots,\pi(t)$ are the last $t$ elements of some order of $\G$).

    For $t=1$, before the first iteration, $\bLambda=[d]=\pi[1:d]$. Plugging $m=d$ into \autoref{Prop: sink discrimination iteratively}
    gives that, the node $i_1$ found in the first iteration is a sink in $\G$. Therefore, after the first iteration, $\pi=(i_1)$ is the tail of some order of $\G$.

    Suppose after the $t$-th iteration $(t\in[d-1])$, $\pi$ is the tail of some order $\pi_0\sim\G$. We then prove this for $t+1$. Since before the $t+1$-th iteration (after the $t$-th iteration), $\pi$ is the tail of $\pi_0$, it holds that the current remaining nodes $\bLambda=\pi_0[1:d-t]$. Plugging $m=d-t$ into \autoref{Prop: sink discrimination iteratively}
    gives that, the node $i_{t+1}$ found in the $t+1$-th iteration is a sink in $\G^\bLambda$. Therefore, after the $t+1$-th iteration, $\pi := (i_{t+1}, \pi)$ is the tail of some order of $\G$.

    By induction, for any $t\in[d]$, after the $t$-th iteration, $\pi$ is the tail of some order of $\G$. Setting $t=d$ finishes the proof.
\end{proof}

\subsection{Proof of \autoref{Prop: Assumption coincide}}

\begin{lemma}
    For any random vectors $\X, \Y, \Z$, it holds that
    \begin{align*}
        I(\Y;\Z\given\X) = h(\Y\given\X) - h(\X,\Y,\Z) + h(\X,\Z).
    \end{align*}
    \label{Lemma: appendix_Assumption_coincide}
\end{lemma}
\begin{proof}
    $I(\Y;\Z\given\X) = h(\Y\given\X) - h(\Y\given\X,\Z) = h(\Y\given\X) - h(\X,\Y,\Z) + h(\X,\Z)$.
\end{proof}

\newtheorem*{proposition33}{\textbf{Proposition 3.3}}
\begin{proposition33}
    Under the linear and Gaussian setting, \autoref{Assumption: OLEM} reduces to \autoref{Assumption: LISTEN}, and \autoref{Assumption: CaPS} implies \autoref{Assumption: OLEM}.
\end{proposition33}

\begin{proof}
    Consider any random vector $\V$ generated by the model $\mathcal M=(\G, \{f_i\}_i, \{\varepsilon_i\}_i)$ under the linear and Gaussian setting:
    \begin{align*}
        f_i(\Pa_i)=\sum_{V_j\in\Pa_i}w_{ji}V_j,\quad\varepsilon_i\sim\mathcal N(0,\sigma_i^2).
    \end{align*}
    For any $\pi\sim\G,~m\in[d]$, and for any $i,j\in\pi[1:m]$ such that $i$ is a sink in $\G^{\pi[1:m]}$, while $j$ is not, we show that the inequality in \autoref{Assumption: OLEM} reduces to that in \autoref{Assumption: LISTEN}.

    Denote by $\bSigma$ the covariance matrix of $\V^{\pi[1:m]}$, by $|\bSigma|$ the determinant of $\bSigma$, and by $M_{jj}$ the determinant of the sub-matrix obtained by removing the row and column corresponding to $V_j$ in $\bSigma$. By \autoref{Lemma: appendix_Assumption_coincide}, it holds that
    \begin{align*}
        I(V_j; \Des_j^{\pi[1:m]}\given\W_j^{\pi[1:m]}) = h(V_j\given\W_j^{\pi[1:m]}) - h(\W_j^{\pi[1:m]}, V_j, \Des_j^{\pi[1:m]}) + h(\W_j^{\pi[1:m]}, \Des_j^{\pi[1:m]}).
    \end{align*}
    From the proof of \autoref{Prop: sink discrimination iteratively}
    we have
    \begin{align*}
        h(V_j\given\W_j^{\pi[1:m]}) = h(\varepsilon_j) = (1/2)\ln(2\pi\e\sigma_j^2).
    \end{align*}
    Also, noting that $\V^{\pi[1:m]} = \W_j^{\pi[1:m]}\cup\{V_j\}\cup\Des_j^{\pi[1:m]}$, we have
    \begin{align*}
        &h(\W_j^{\pi[1:m]}, V_j, \Des_j^{\pi[1:m]}) = (1/2)\ln((2\pi\e)^{m}|\bSigma|),\\
        &h(\W_j^{\pi[1:m]}, \Des_j^{\pi[1:m]}) = (1/2)\ln((2\pi\e)^{m-1}M_{jj}).
    \end{align*}
    Therefore it holds that
    \begin{align*}
        I(V_j; \Des_j^{\pi[1:m]}\given\W_j^{\pi[1:m]}) = \frac12\ln\frac{\sigma_j^2M_{jj}}{|\bSigma|}.
    \end{align*}
    Denote by $\bOmega$ the inverse matrix of $\bSigma$, and by $\bSigma^*$ the adjugate matrix of $\bSigma$. Let $\Omega_{jj}$ denote the diagonal element of $\bOmega$ corresponding to $V_j$, and define $\bSigma^*_{jj}$ similarly. Note that $\bOmega = \bSigma^*/|\bSigma|$. This gives $\Omega_{jj} = \Sigma^*_{jj}/|\bSigma| = M_{jj}/|\bSigma|$. Plugging this into the equation above gives
    \begin{align*}
        I(V_j; \Des_j^{\pi[1:m]}\given\W_j^{\pi[1:m]}) = \frac12\ln(\sigma_j^2\Omega_{jj}).
    \end{align*}
    By Prop.2 of \citet{aistats/LISTEN2018}, $\Omega_{jj} = \sigma_j^{-2} + \sum_{l:V_j\in\Pa_l^{\pi[1:m]}}\sigma_l^{-2}w_{jl}^2$. This gives 
    \begin{align*}
        I(V_j; \Des_j^{\pi[1:m]}\given\W_j^{\pi[1:m]}) > h(\varepsilon_j) - h(\varepsilon_i)\quad&\text{iff}\quad \frac12\ln(\sigma_j^2\sigma_i^{-2}) < \frac12\ln\sigma_j^2\left(\sigma_j^{-2} + \sum_{l:V_j\in\Pa_l^{\pi[1:m]}}\sigma_l^{-2}w_{jl}^2\right)\\
        &\text{iff}\quad\sigma_i^{-2} < \sigma_j^{-2} + \sum_{l:V_j\in\Pa_l^{\pi[1:m]}}\sigma_l^{-2}w_{jl}^2.
    \end{align*}
    Therefore, \autoref{Assumption: OLEM} reduces to \autoref{Assumption: LISTEN} under the linear Gaussian setting.

    Then we prove \autoref{Assumption: CaPS} implies \autoref{Assumption: LISTEN}. For any $\pi\sim\G,~m\in[d]$, suppose (1) of \autoref{Assumption: CaPS} holds, then for any $i,j\in\pi[1:m]$ such that $i$ is a sink in $\G^{\pi[1:m]}$, while $j$ is not, we have $\sigma_i > \sigma_j$, and this gives
    \begin{align*}
        \sigma_i^{-2} < \sigma_j^{-2} \leq \sigma_j^{-2} + \sum_{l:V_j\in\Pa_l^{\pi[1:m]}}\sigma_l^{-2}w_{jl}^2.
    \end{align*}
    Also, suppose (2) of \autoref{Assumption: CaPS} holds, then for any $i,j\in\pi[1:m]$ such that $i$ is a sink in $\G^{\pi[1:m]}$, while $j$ is not, noting that $\frac{\partial f_l(\pa_l)}{\partial v_j}=w_{jl}$, it holds that
    \begin{align*}
        \sigma_i^{-2} \leq \sigma_{\min}^{-2} < \sigma_j^{-2} + \sum_{l:V_j\in\Pa_l^{\pi[1:m]}}\sigma_l^{-2}w_{jl}^2.
    \end{align*}
    Therefore, \autoref{Assumption: CaPS} implies \autoref{Assumption: LISTEN} under the linear Gaussian setting, and it follows that \autoref{Assumption: CaPS} implies \autoref{Assumption: OLEM} under the linear Gaussian setting.
\end{proof}

\section{Experiments}

\label{Appendix: experiments}

All experiments were run on a Nvidia RTX 6000 GPU and two Intel Xeon w7-3455 CPUs.

\subsection{Implementation Details}

\textbf{Order Learning.}
As presented in \autoref{Alg: OLEM}, OLEM recursively learns the order from the joint distribution, where the joint distribution is only used to compute the conditional entropy terms. Since we only have access to observational data in practice, we estimate the required conditional entropy terms using the Kozachenko–Leonenko estimator \citep{entropy_estimate_kl(1nn)}. Notably, OLEM is parameter-free and requires no hyperparameter tuning.

\textbf{AUF Decision-making.}
OLEM-Rh leverages the conditional flow model \citep{conditional_flows_winkler2019} to construct the post-decision distribution sampler, where the flow model is implemented using the conditional invertible neural network (cINN) \citep{conditional_flows_ardizzone2019}. Specifically, the cINN in our experiment consists of 16 sequentially chained blocks, and is trained using the Adam optimizer \citep{iclr/adam2015} with a learning rate of 0.001 and a batch size of 256. During the training process, we employ a 70/30 train-validation split and use the validation set for early stopping. Detailed architectural and training specifications of the cINN follow \citet{FrEIA_ardizzone2018, conditional_flows_ardizzone2019}.

With the constructed sampler, OLEM-Rh solves the problem \eqref{eq: AUF_reduced} using Adam, whose learning rate is set to 0.5 in our experiments. The size $n$ of the samples drawn from the constructed sampler is a hyperparameter, and is fixed at $n=1000$ across all experiments to ensure consistency.

\subsection{Additional Synthetic Results on Order Learning}

We present additional results of the synthetic experiments in \autoref{Subsection: expt_order}. The noise variables follow beta, Gaussian, or exponential distributions. We vary the dimensionality $d\in\{5, 10, 20\}$ and the edge density $p\in\{0.3, 0.5, 0.8\}$. Although the results for $d=10,~p=0.3$ with beta or Gaussian noise are shown in \autoref{Fig: beta_gauss_10_0d3}, we reproduce them here to facilitate direct comparison. The setting $d=20, p=0.8$ is omitted due to the excessive number of edges.

\begin{figure}[H]
	\centering
	\begin{minipage}{0.49\linewidth}
		\centering
		\includegraphics[width=\linewidth]{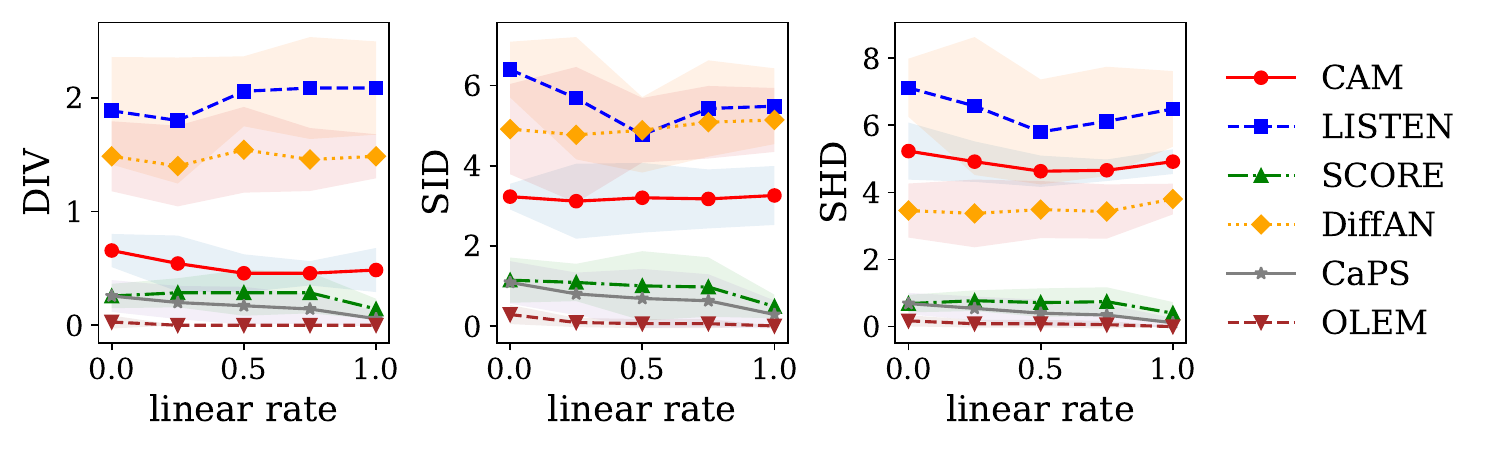}
		\caption{$d=5,~p=0.3$ with beta noise}
		\label{Fig: beta_5_0d3}
	\end{minipage}
	\hfill
	\begin{minipage}{0.49\linewidth}
		\centering
		\includegraphics[width=\linewidth]{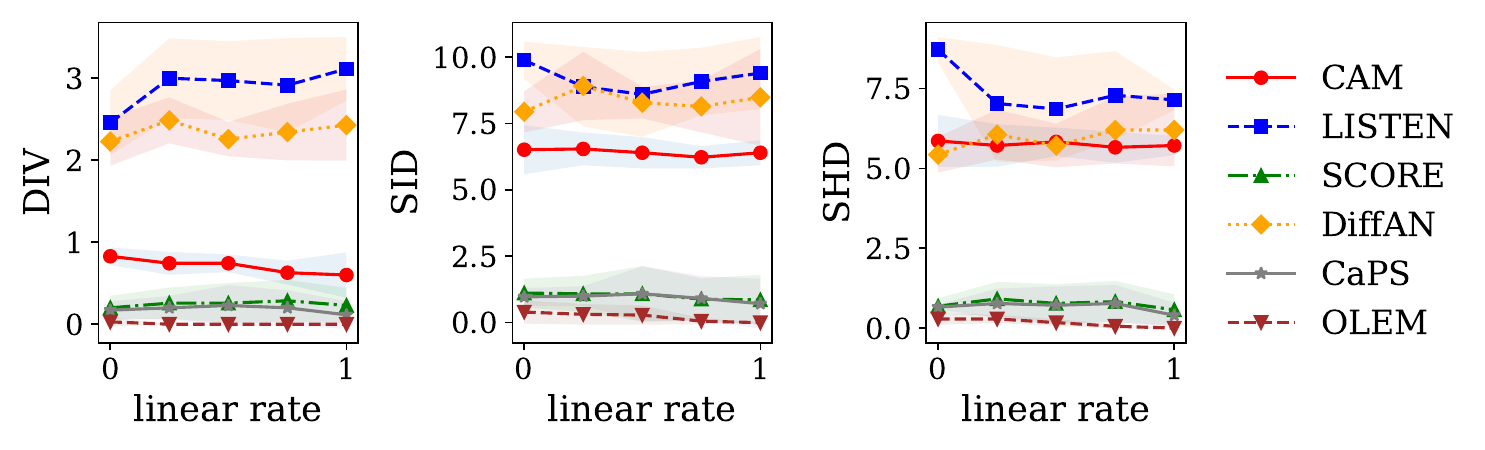}
		\caption{$d=5,~p=0.5$ with beta noise}
		\label{Fig: beta_5_0d5}
	\end{minipage}
\end{figure}

\begin{figure}[H]
	\centering
	\begin{minipage}{0.49\linewidth}
		\centering
		\includegraphics[width=\linewidth]{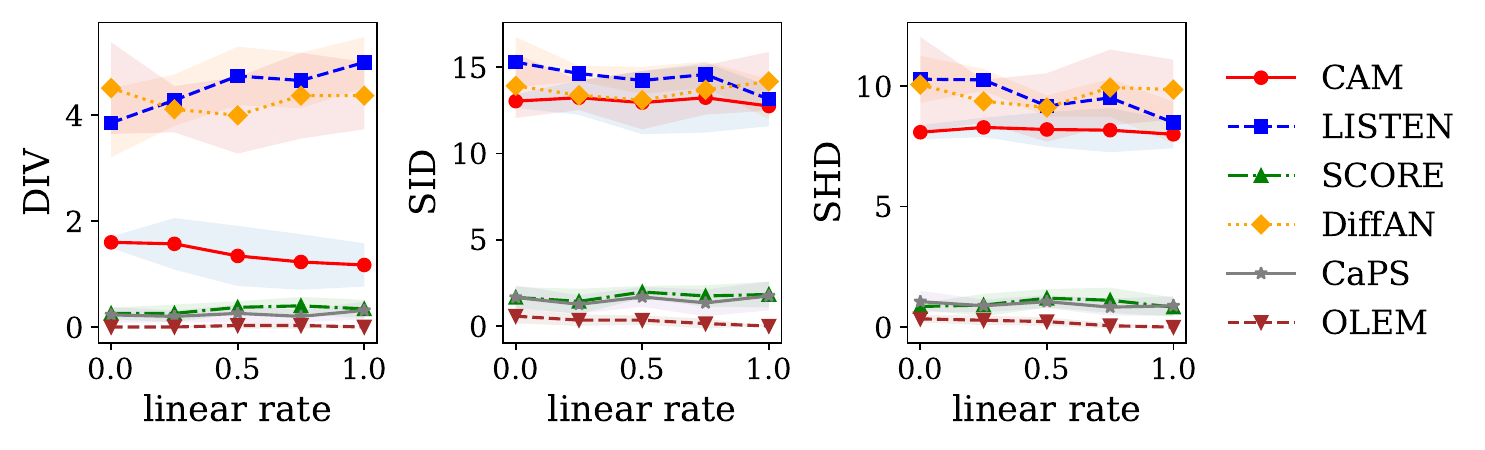}
		\caption{$d=5,~p=0.8$ with beta noise}
		\label{Fig: beta_5_0d8}
	\end{minipage}
	\hfill
	\begin{minipage}{0.49\linewidth}
		\centering
		\includegraphics[width=\linewidth]{figs/beta_10_0d3_tight.pdf}
		\caption{$d=10,~p=0.3$ with beta noise}
		\label{Fig: beta_10_0d3}
	\end{minipage}
\end{figure}

\begin{figure}[H]
	\centering
	\begin{minipage}{0.49\linewidth}
		\centering
		\includegraphics[width=\linewidth]{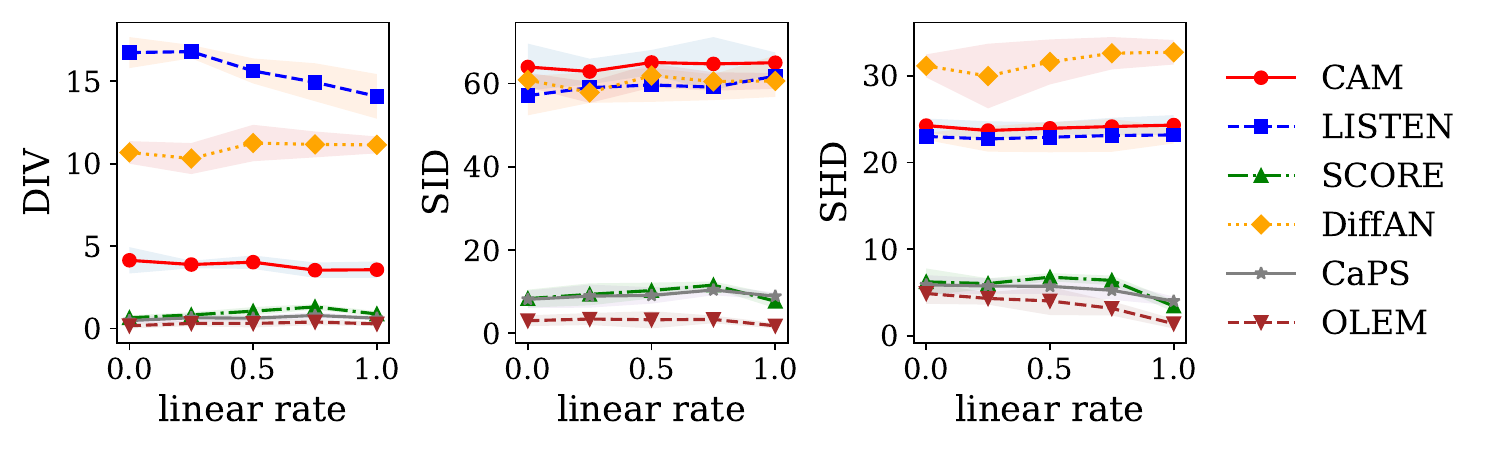}
		\caption{$d=10,~p=0.5$ with beta noise}
		\label{Fig: beta_10_0d5}
	\end{minipage}
	\hfill
	\begin{minipage}{0.49\linewidth}
		\centering
		\includegraphics[width=\linewidth]{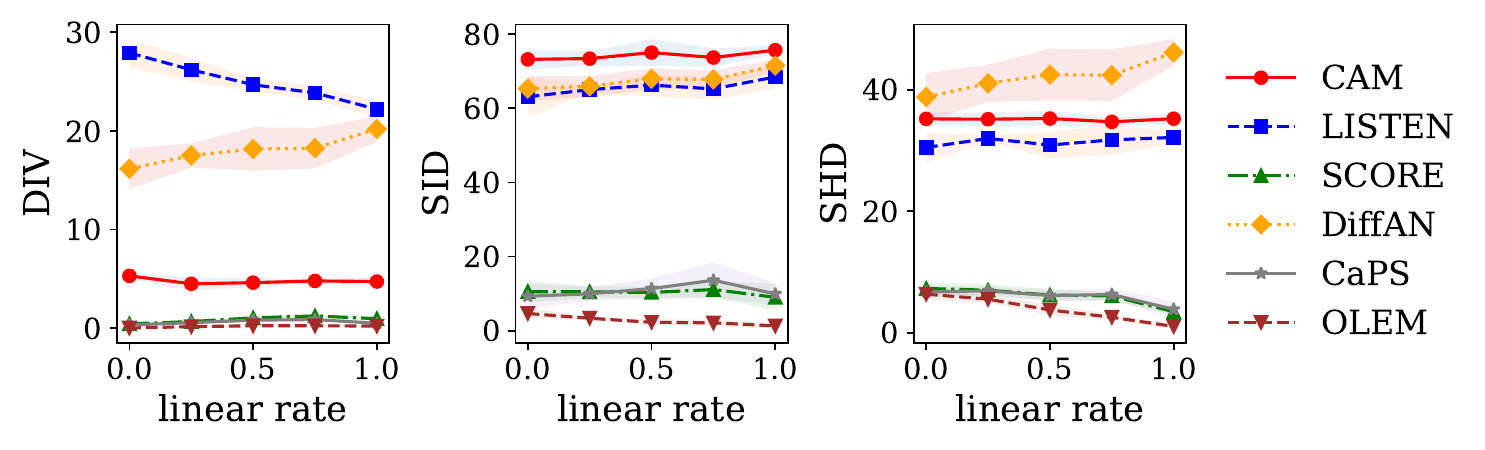}
		\caption{$d=10,~p=0.8$ with beta noise}
		\label{Fig: beta_10_0d8}
	\end{minipage}
\end{figure}

\begin{figure}[H]
	\centering
	\begin{minipage}{0.49\linewidth}
		\centering
		\includegraphics[width=\linewidth]{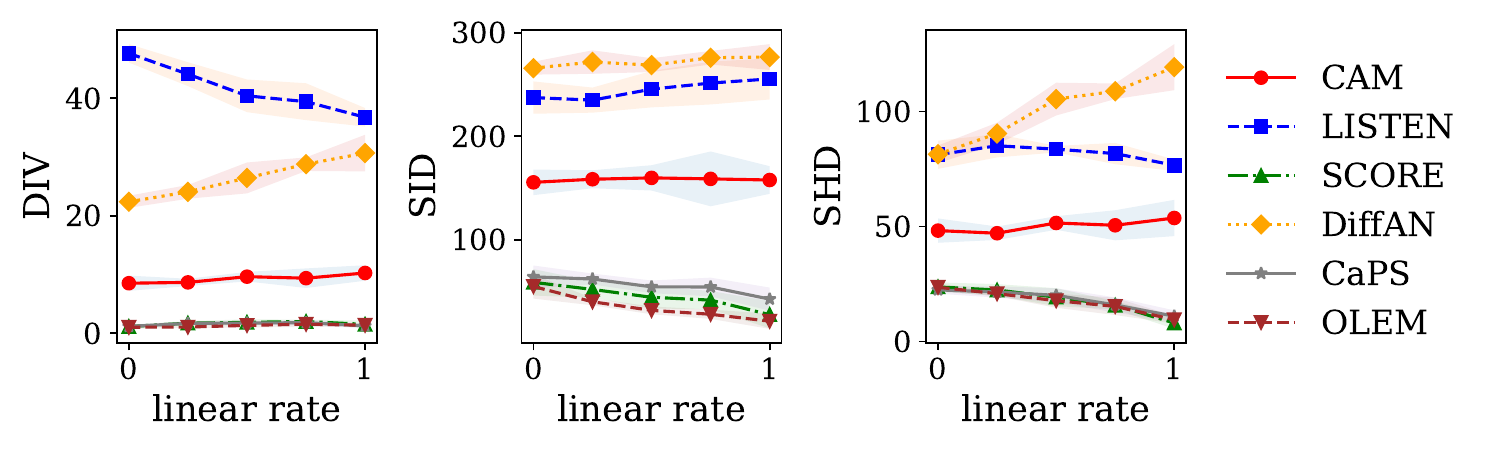}
		\caption{$d=20,~p=0.3$ with beta noise}
		\label{Fig: beta_20_0d3}
	\end{minipage}
	\hfill
	\begin{minipage}{0.49\linewidth}
		\centering
		\includegraphics[width=\linewidth]{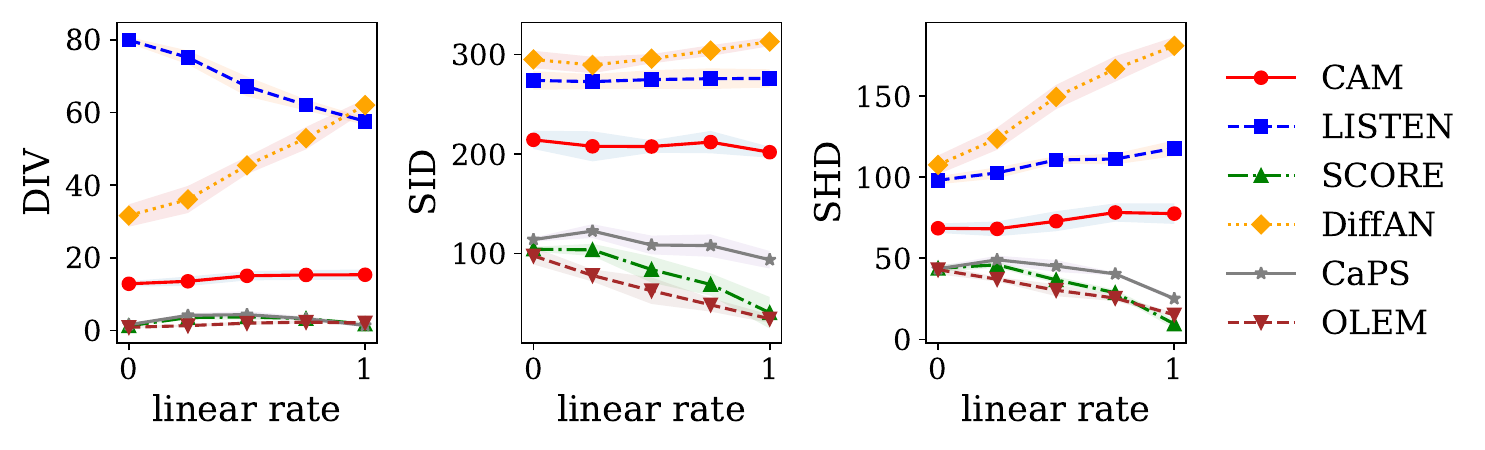}
		\caption{$d=20,~p=0.5$ with beta noise}
		\label{Fig: beta_20_0d5}
	\end{minipage}
\end{figure}

\begin{figure}[H]
	\centering
	\begin{minipage}{0.49\linewidth}
		\centering
		\includegraphics[width=\linewidth]{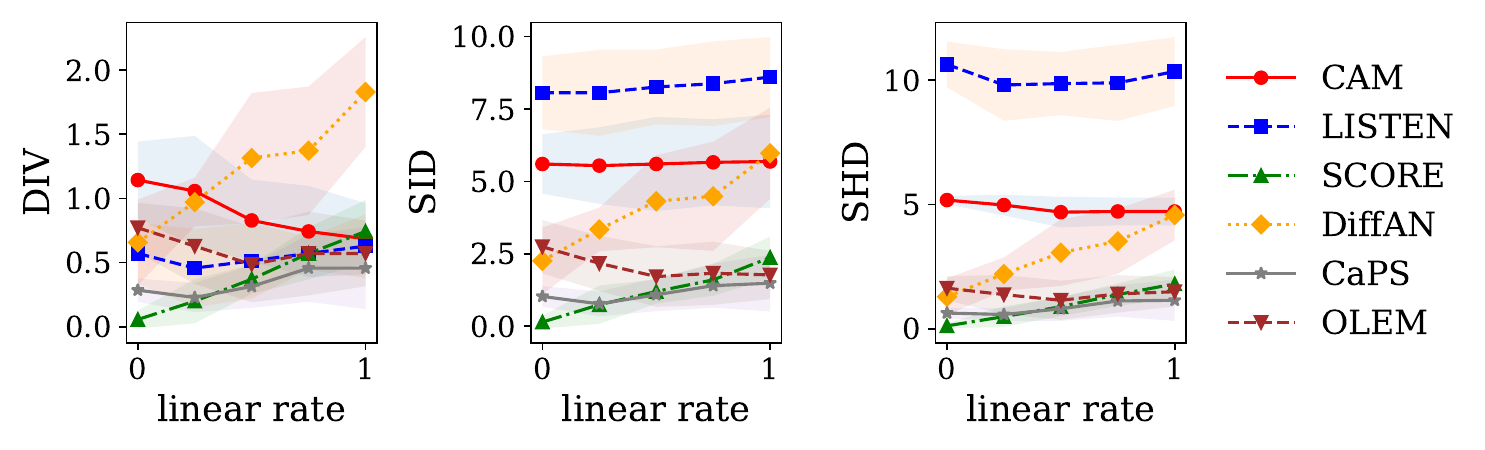}
		\caption{$d=5,~p=0.3$ with Gaussian noise}
		\label{Fig: gauss_5_0d3}
	\end{minipage}
	\hfill
	\begin{minipage}{0.49\linewidth}
		\centering
		\includegraphics[width=\linewidth]{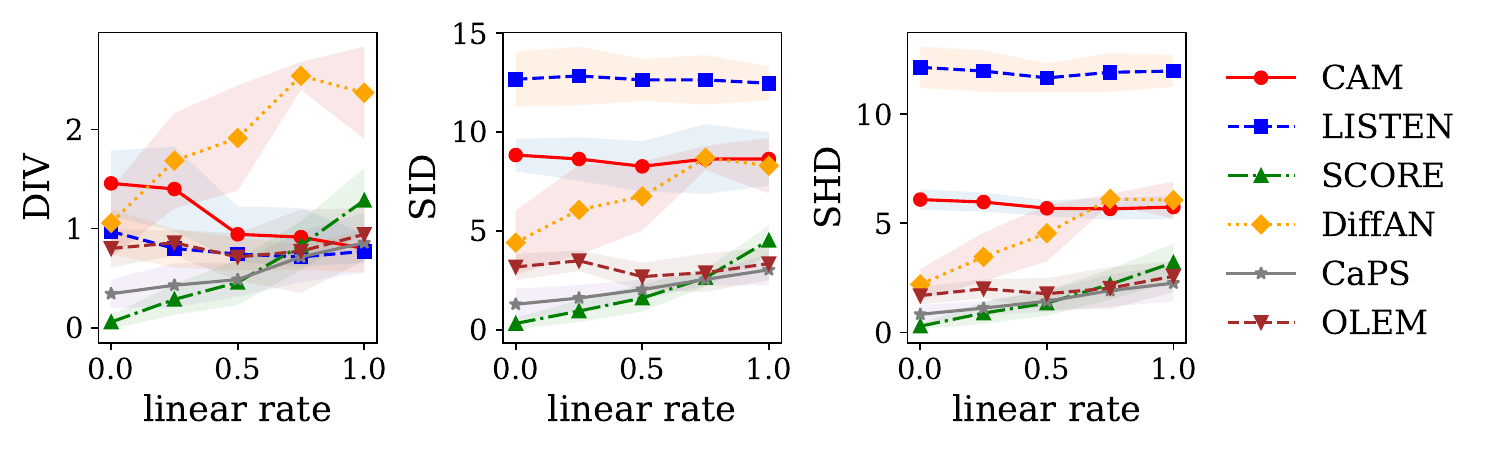}
		\caption{$d=5,~p=0.5$ with Gaussian noise}
		\label{Fig: gauss_5_0d5}
	\end{minipage}
\end{figure}

\begin{figure}[H]
	\centering
	\begin{minipage}{0.49\linewidth}
		\centering
		\includegraphics[width=\linewidth]{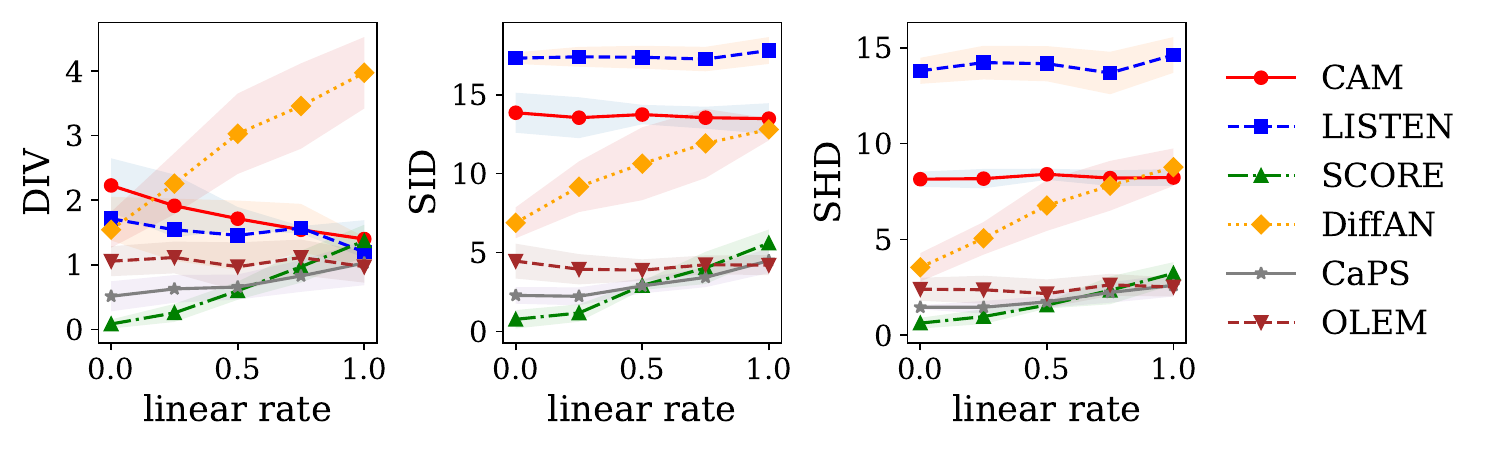}
		\caption{$d=5,~p=0.8$ with Gaussian noise}
		\label{Fig: gauss_5_0d8}
	\end{minipage}
	\hfill
	\begin{minipage}{0.49\linewidth}
		\centering
		\includegraphics[width=\linewidth]{figs/gauss_10_0d3_tight.pdf}
		\caption{$d=10,~p=0.3$ with Gaussian noise}
		\label{Fig: gauss_10_0d3}
	\end{minipage}
\end{figure}

\begin{figure}[H]
	\centering
	\begin{minipage}{0.49\linewidth}
		\centering
		\includegraphics[width=\linewidth]{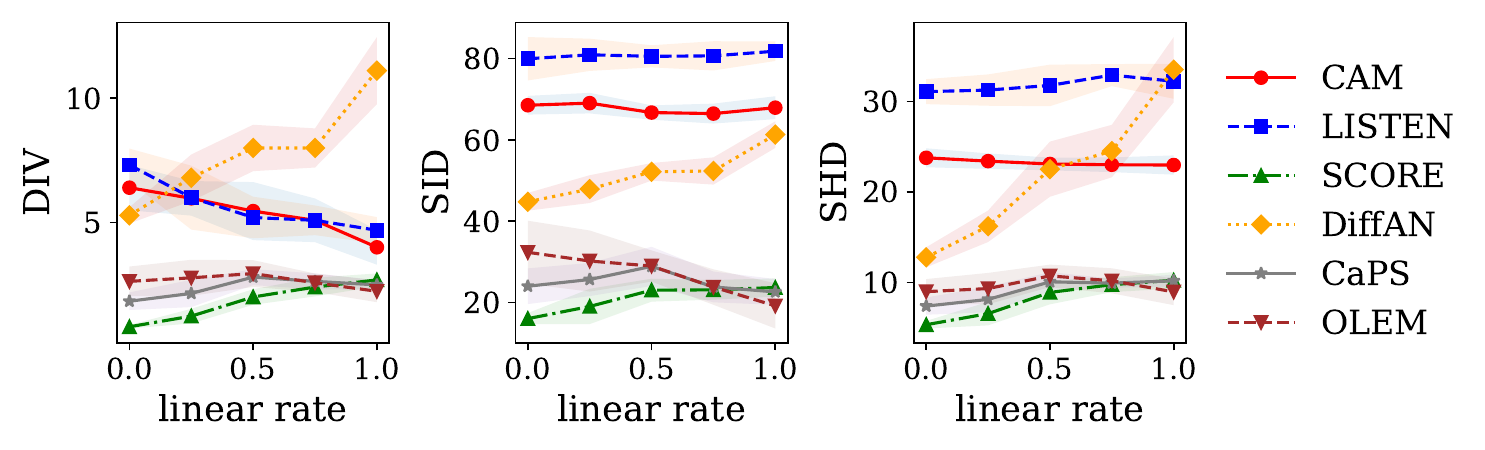}
		\caption{$d=10,~p=0.5$ with Gaussian noise}
		\label{Fig: gauss_10_0d5}
	\end{minipage}
	\hfill
	\begin{minipage}{0.49\linewidth}
		\centering
		\includegraphics[width=\linewidth]{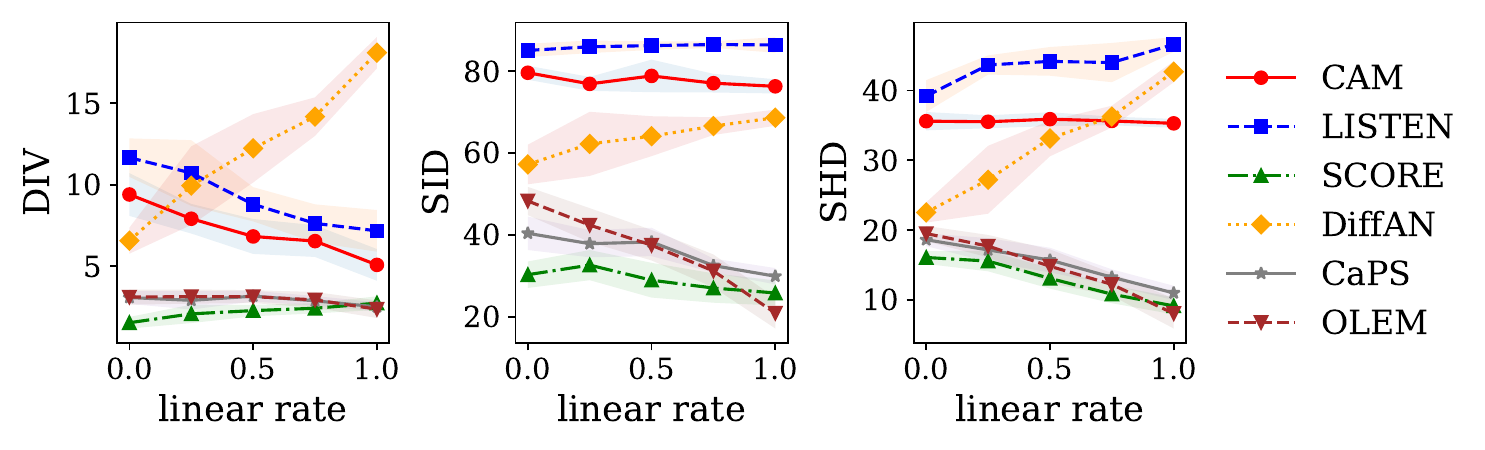}
		\caption{$d=10,~p=0.8$ with Gaussian noise}
		\label{Fig: gauss_10_0d8}
	\end{minipage}
\end{figure}

\begin{figure}[H]
	\centering
	\begin{minipage}{0.49\linewidth}
		\centering
		\includegraphics[width=\linewidth]{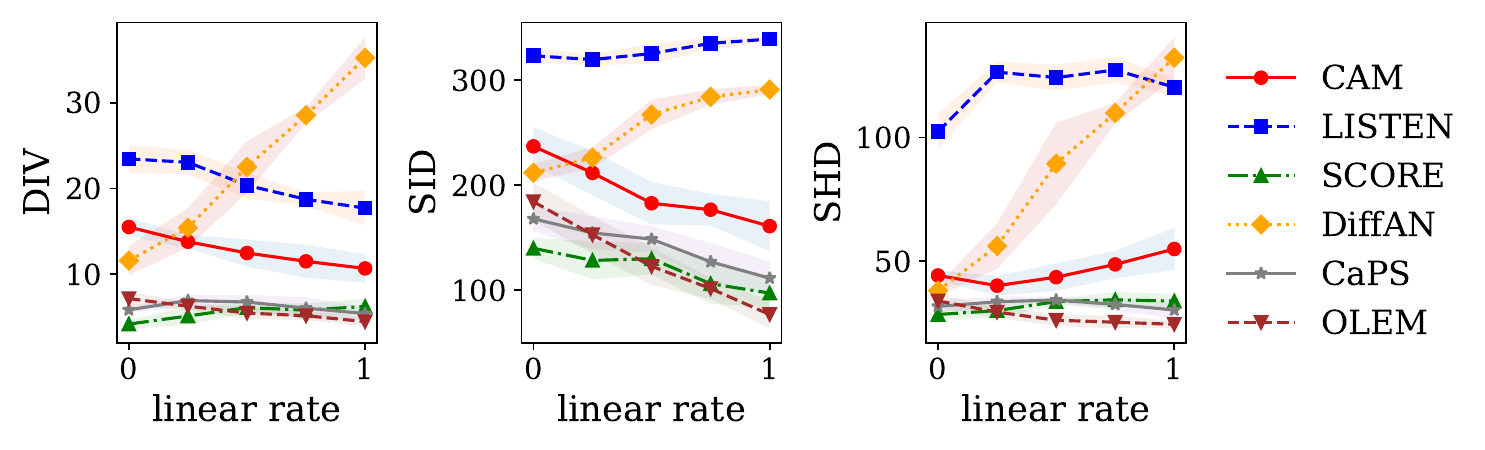}
		\caption{$d=20,~p=0.3$ with Gaussian noise}
		\label{Fig: gauss_20_0d3}
	\end{minipage}
	\hfill
	\begin{minipage}{0.49\linewidth}
		\centering
		\includegraphics[width=\linewidth]{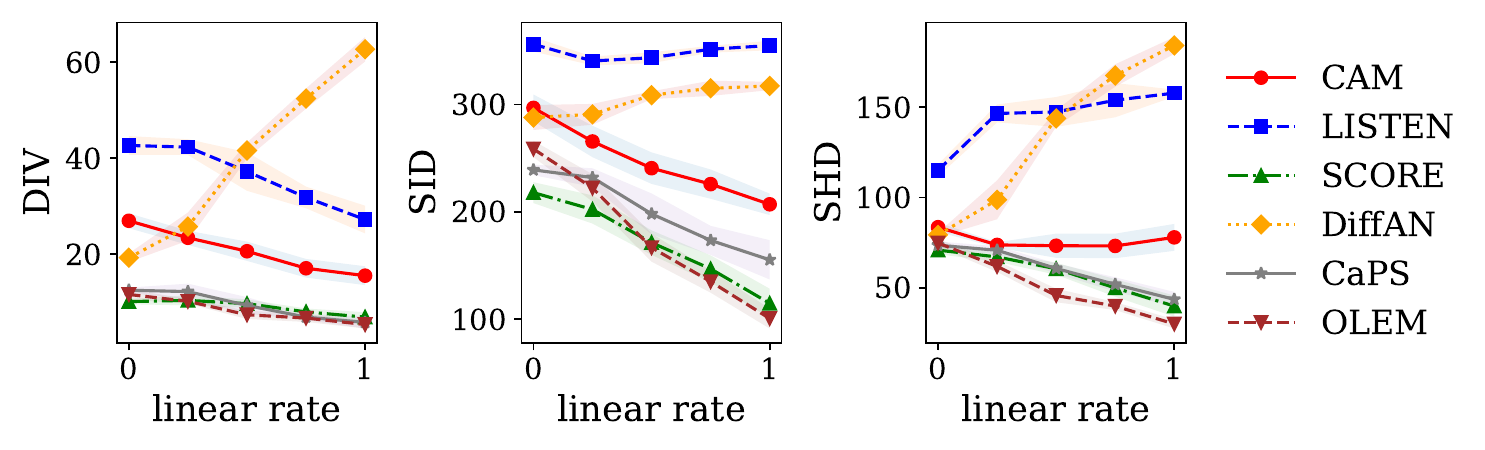}
		\caption{$d=20,~p=0.5$ with Gaussian noise}
		\label{Fig: gauss_20_0d5}
	\end{minipage}
\end{figure}

\begin{figure}[H]
	\centering
	\begin{minipage}{0.49\linewidth}
		\centering
		\includegraphics[width=\linewidth]{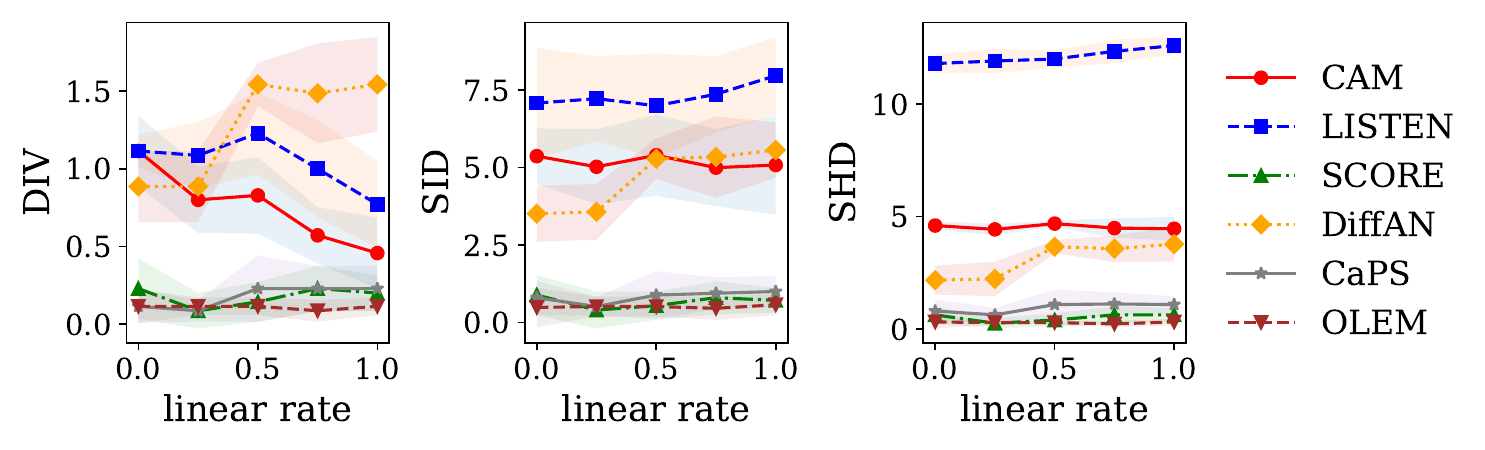}
		\caption{$d=5,~p=0.3$ with exponential noise}
		\label{Fig: exp_5_0d3}
	\end{minipage}
	\hfill
	\begin{minipage}{0.49\linewidth}
		\centering
		\includegraphics[width=\linewidth]{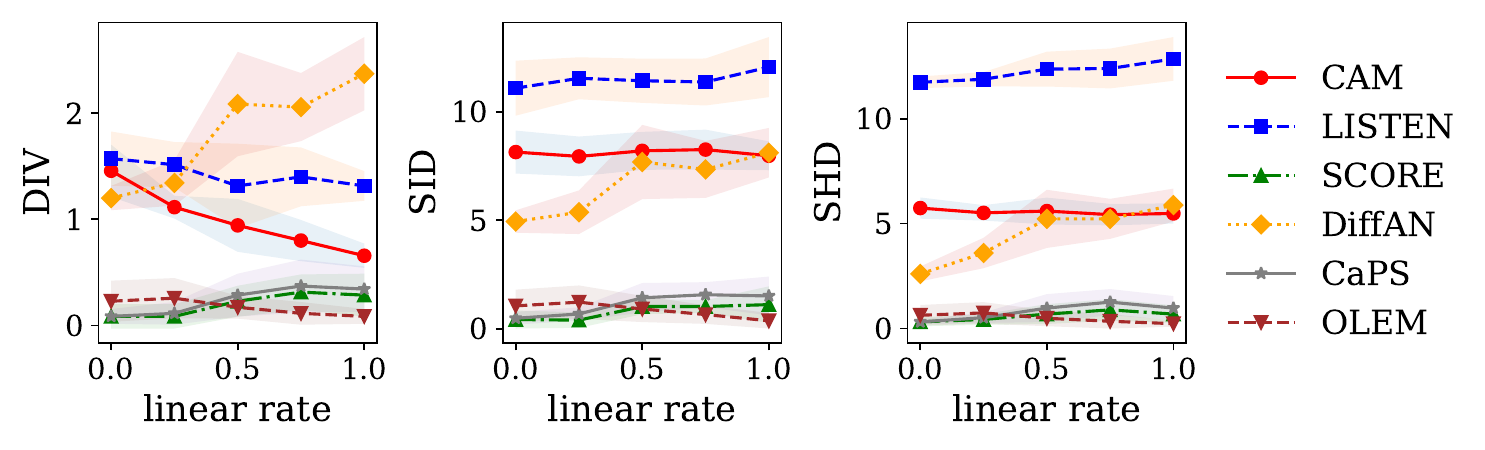}
		\caption{$d=5,~p=0.5$ with exponential noise}
		\label{Fig: exp_5_0d5}
	\end{minipage}
\end{figure}

\begin{figure}[H]
	\centering
	\begin{minipage}{0.49\linewidth}
		\centering
		\includegraphics[width=\linewidth]{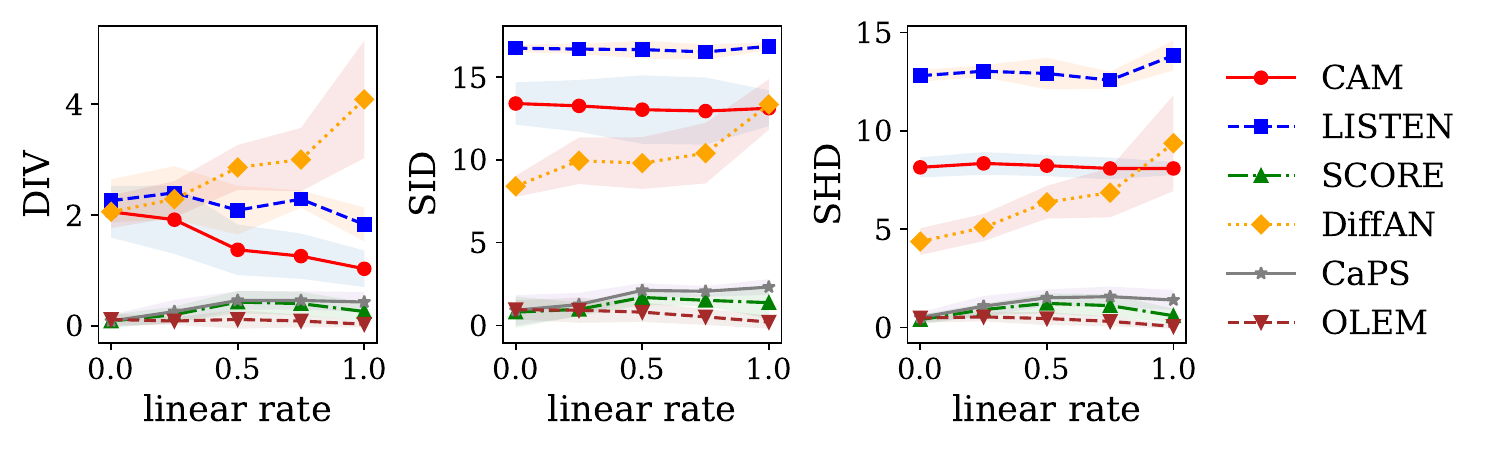}
		\caption{$d=5,~p=0.8$ with exponential noise}
		\label{Fig: exp_5_0d8}
	\end{minipage}
	\hfill
	\begin{minipage}{0.49\linewidth}
		\centering
		\includegraphics[width=\linewidth]{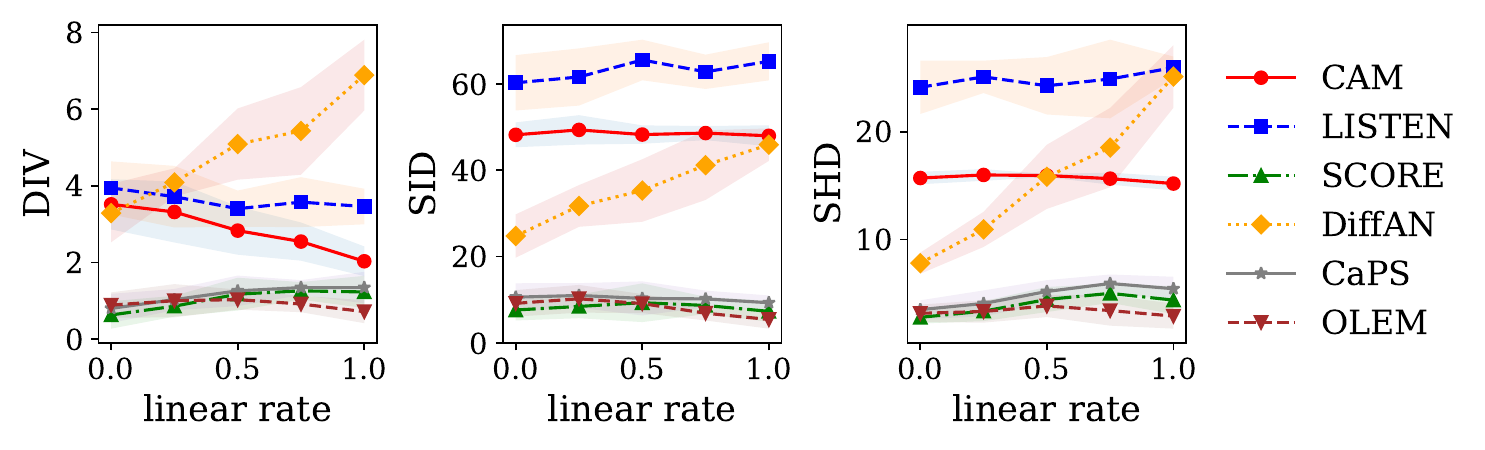}
		\caption{$d=10,~p=0.3$ with exponential noise}
		\label{Fig: exp_10_0d3}
	\end{minipage}
\end{figure}

\begin{figure}[H]
	\centering
	\begin{minipage}{0.49\linewidth}
		\centering
		\includegraphics[width=\linewidth]{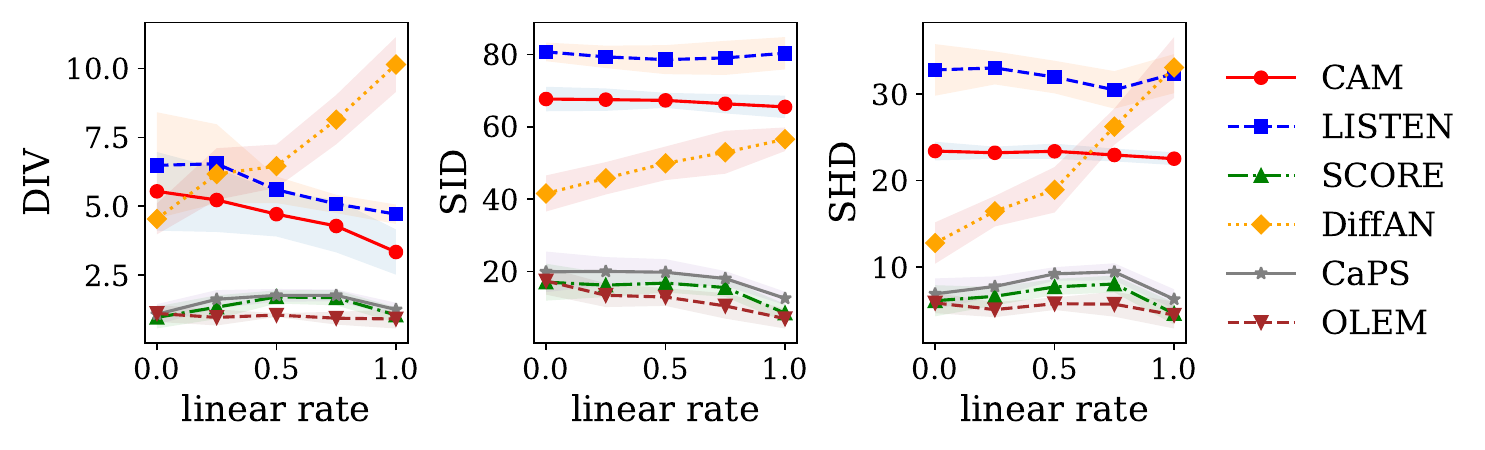}
		\caption{$d=10,~p=0.5$ with exponential noise}
		\label{Fig: exp_10_0d5}
	\end{minipage}
	\hfill
	\begin{minipage}{0.49\linewidth}
		\centering
		\includegraphics[width=\linewidth]{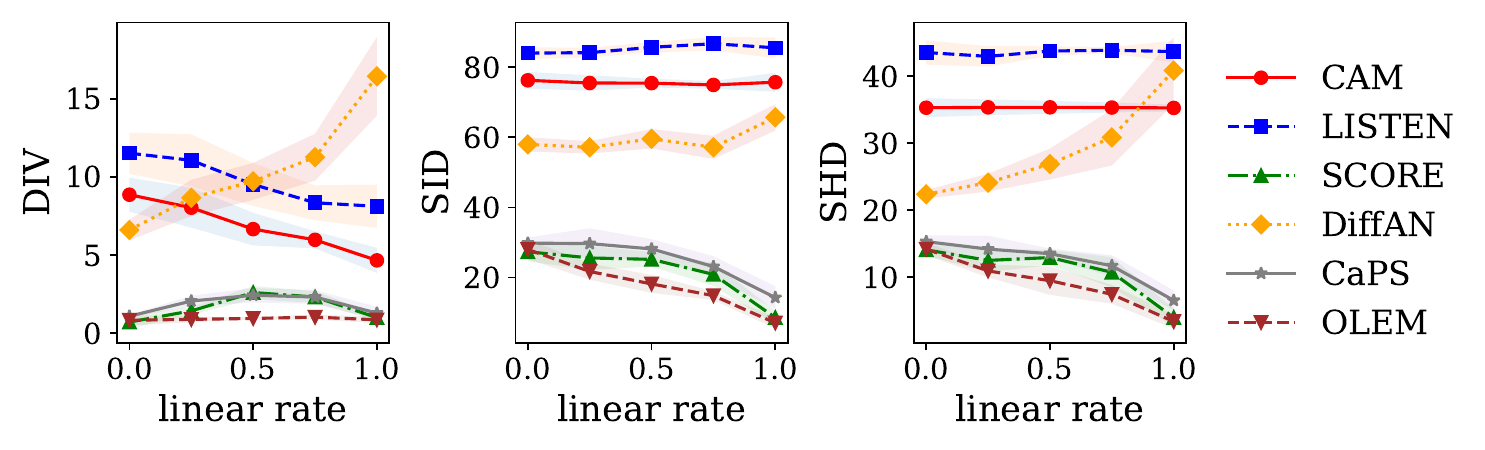}
		\caption{$d=10,~p=0.8$ with exponential noise}
		\label{Fig: exp_10_0d8}
	\end{minipage}
\end{figure}

\begin{figure}[H]
	\centering
	\begin{minipage}{0.49\linewidth}
		\centering
		\includegraphics[width=\linewidth]{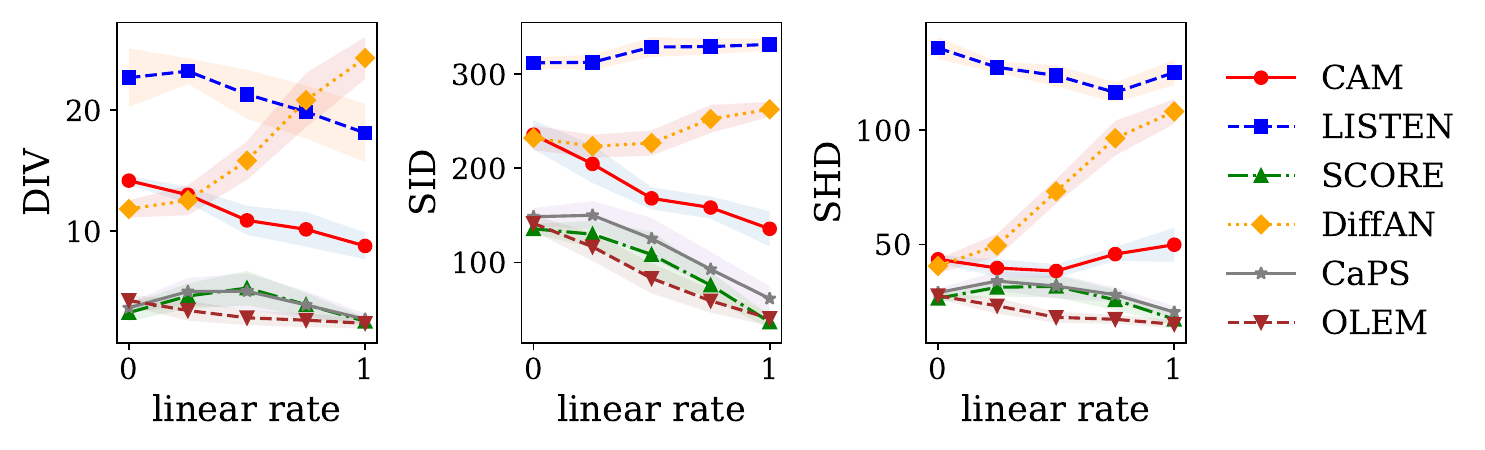}
		\caption{$d=20,~p=0.3$ with exponential noise}
		\label{Fig: exp_20_0d3}
	\end{minipage}
	\hfill
	\begin{minipage}{0.49\linewidth}
		\centering
		\includegraphics[width=\linewidth]{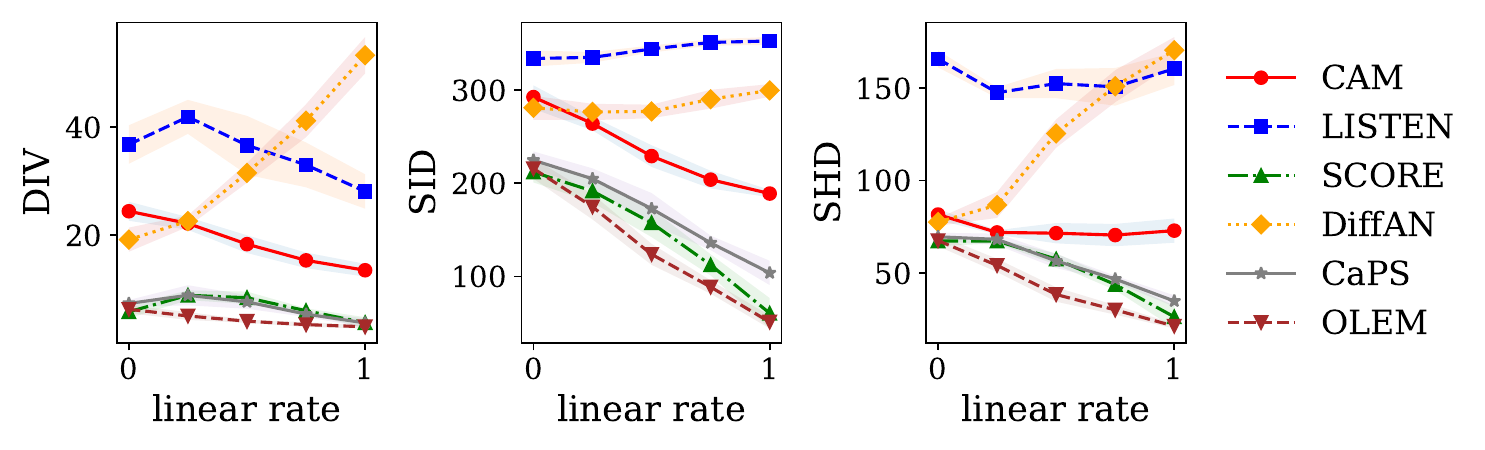}
		\caption{$d=20,~p=0.5$ with exponential noise}
		\label{Fig: exp_20_0d5}
	\end{minipage}
\end{figure}

\end{document}